\documentclass[12pt]{article}

\usepackage{amsmath}
\usepackage{amsthm}
\usepackage{amssymb}
\usepackage{commath}              
\usepackage{graphicx}
\usepackage[margin=2.5cm]{geometry}
\usepackage{setspace}
\usepackage{color}
\usepackage{xparse}
\usepackage{algorithm}
\usepackage{algpseudocode}
\usepackage{subfig}
\usepackage{array}
\usepackage{endnotes}
\usepackage{fixltx2e}
\usepackage{pdflscape}
\usepackage[page]{appendix}	

\newtheorem{theorem}{Theorem}
\newtheorem{lemma}[theorem]{Lemma}
\newtheorem{corollary}[theorem]{Corollary}
\newtheorem{definition}[theorem]{Definition} 
\newtheorem{remark}{Remark} 

\DeclareMathOperator{\diag}{diag}

\DeclareMathOperator{\NLM}{NLM}
\DeclareMathOperator{\PSNR}{PSNR}
\DeclareMathOperator{\SNR}{SNR}

\DeclareMathOperator{\EIG}{EIG}
\DeclareMathOperator{\COL}{COL}

\DeclareMathOperator{\IMAGE}{IMAGE}
\DeclareMathOperator{\ChebCoef}{CHEB-COEF}

\DeclareMathOperator{\NLMSB}{NLM-SB}

\newcolumntype{C}{>{\centering\arraybackslash} m{2.5cm} }

\usepackage{authblk}

\title{An algorithm for improving Non-Local Means operators via low-rank approximation}
\author[1]{Victor May\thanks{mayvic@gmail.com} }
\author[2]{Yosi Keller\thanks{yosi.keller@gmail.com} }
\author[1]{Nir Sharon\thanks{nir.sharon@math.tau.ac.il}}
\author[1]{Yoel Shkolnisky\thanks{yoelsh@post.tau.ac.il} }
\affil[1]{School of Mathematical Sciences, Tel-Aviv University, Tel-Aviv, Israel}
\affil[2]{Faculty of Engineering, Bar-Ilan University, Ramat-Gan, Israel}

\date{}

\begin{document}

\maketitle

\begin{abstract}
We present a method for improving a Non Local Means operator by computing its low-rank approximation. The low-rank operator is constructed by applying a filter to the spectrum of the original Non Local Means operator. This results in an operator which is less sensitive to noise while preserving important properties of the original operator. The method is efficiently implemented based on Chebyshev polynomials and is demonstrated on the application of natural images denoising. For this application, we provide a comprehensive comparison of our method with leading denoising methods.
\end{abstract}


\section{Introduction}\label{sec:introduction}

Denoising images is a classical problem in image processing. Representative
approaches for solving this problem include local methods such as linear filtering and
anisotropic diffusion~\cite{perona1990scale}, global methods such as total-variation~\cite{rudin1992nonlinear}
and wavelet shrinkage~\cite{coifman1995translation}, and discrete methods such as Markov Random field
denoising~\cite{szeliski2008comparative} and discrete universal denoiser~\cite{motta2011idude}.

As is often the case with inverse problems, many of the algorithms for image denoising involve priors on the solution. Commonly, the only prior knowledge
assumed about the image is that it is of natural origin. In that case, the sought-after prior should represent the statistics of a natural image. Natural image statistics
is a research topic on its own~\cite{hyvarinen2009natural,zontak2011internal}, having importance
for image processing problems other than denoising: segmentation,
texture synthesis, image inpainting, super-resolution and more. An
important observation in natural image statistics is
that images contain repetitive local patterns. This observation is
the basis of patch-based denoising methods such as Non-Local Means (NLM)~\cite{buades2005non}
and block-matching and 3D filtering (BM3D)~\cite{dabov2009bm3d}.

In NLM, the denoising operator is constructed from patches of the corrupted image itself. As such,
the denoising operator is affected by the noise. It had been noted~\cite{meyer2014perturbation,singer2009diffusion} that larger eigenvalues of the operator
are less sensitive to noise than smaller ones, creating an opportunity to improve the operator. In this work we propose a method for computing such a modified operator, by means of filtering the noisy one with a low-pass filter applied to its eigenvalues. In other words, we pose the problem as a filter design task, which is a classical task in signal processing. Our chosen filter function suppresses eigenvalues with small magnitude, and accurately approximates the remaining low-rank operator, while preserving the fundamental properties of the original operator. The filter is efficiently applied to the NLM operator based on Chebyshev polynomials for matrices.

We further study the concept of low-rank NLM denoising operators by numerical experiments. In these experiments we investigate our method for two main questions. The first is the dependence of the low-rank operator on its two tuning parameters, which are the width of the kernel of the original NLM operator, and the rank of the low-rank operator. The second question is the performance of our method compared to other advanced denoising algorithms. We provide a comprehensive comparison which includes a few popular state-of-the-art methods and a two-stage denoising scheme based on our low-rank operator.

The paper is organized as follows. In Section~\ref{sec:nlmeans} we present the notations, the problem's formulation and the NLM method. In Section~\ref{sec:spectrum shaping} we introduce our method of constructing low-rank approximations for NLM operators, which is based on matrix filtering functions. Next, in Section~\ref{sec:computation}, we show how to efficiently implement this low-rank approximation using Chebyshev polynomials. Section~\ref{sec:experiments} provides numerical experiments with our algorithm, where we discuss its parameters and compare it to other advanced denoising methods. We summarize the paper in Section~\ref{sec:conclusions}.

\section{Preliminaries}  \label{sec:nlmeans}

We start by introducing the notations and the model for the denoising problem. Let $I$ be a discrete grid of dimension $d$ (typically, $d=1$ or $d=2$) and denote by $X=\left\{x_i \mid i \in I \right\}$ a clean signal defined on $I$. Let $Y=\left\{y_i \mid i \in I \right\}$ be a signal contaminated with noise, that is $y_i=x_i+r_i$, where $r_i\sim N\left(0,\sigma^{2}\right)$, $i \in I$, are independently identically distributed Gaussian random variables. The goal is to estimate $X$ given $Y$.

The non-local means (NLM) estimator is given as follows. Denote by $N_{i}^Y=N_{i}^Y(d,p)$ the
indices of the pixels within a hypercube of dimension $d$ and side length $p$ centered around the pixel $y_i \in Y$. Let
$v\left(N_{i}^Y\right)$, $i \in I$ be the values of the pixels of $Y$ at the indices $N_{i}^Y$, treated as a column vector. The NLM algorithm
estimates $x_{i}$ as
\[
\hat{x}_{i}=\frac{1}{z_i}\sum_{y_{j}\in Y}K_{h}\left(v\left(N_{i}^Y\right)-v\left(N_{j}^Y\right)\right)y_{j},
\]
where $K_{h}\left(\cdot \right)=\exp({-\frac{\left\Vert \cdot \right\Vert _{2}^{2}}{2h^{2}}})$ is the Gaussian kernel function with width $h>0$, and
\[ z_i=\sum_{y_{j}\in Y}K_h\left(v\left(N_{i}^Y\right)-v\left(N_{j}^Y\right)\right) \]
is a normalization factor.

In matrix notation the NLM estimator is written as
\begin{equation}\label{eq:A}
\hat{x}=Ay,
\end{equation}
where $y \in \mathbb{R}^{n \times 1}$ is the original noisy signal $Y$ represented as a vector, $\hat{x}$ is the denoised signal (the estimations vector), and $A$ is a NLM operator of $Y$, defined as follows.
\begin{definition}\label{def:nlmeans_op}
A matrix $A \in \mathbb{R}^{n \times n}$ is a NLM operator of $Y$ if $A=D^{-1}W$, where
\[ W_{ij}=K_{h}\left(v\left(N_{i}^Y\right)-v\left(N_{j}^Y\right)\right) \]
with $K_{h}\left(\cdot \right)=\exp({-\frac{\left\Vert \cdot \right\Vert _{2}^{2}}{2h^{2}}})$ and $h>0$, and $D$ is a diagonal matrix given by
$D_{ii}=\sum_{j=1}^{n}W_{ij}$.
\end{definition}

\begin{remark}
The pseudocode of all algorithms described in this paper is given in Appendix~\ref{sec:apx_algorithms}. In these algorithms, we denote by $\NLM_{p,h}(Y)$ the NLM operator of image $Y$ with patch size $p$ and kernel width $h$.
\end{remark}

\section{Denoising the NLM operator}\label{sec:spectrum shaping}

As the NLM operator $A$ in~\eqref{eq:A} is constructed from noisy data, we would like to replace it with an operator which is less noisy. One way is to replace $A$ with its low-rank approximation \cite{meyer2014perturbation,singer2009diffusion}. A low-rank approximation obtained by a truncated singular-value
decomposition (SVD) is optimal under the $L_2$ norm~\cite[Chapter 2.5]{golub2012matrix}, however, directly computing the SVD of $A$ is computationally expensive and often impractical. Computing a truncated eigenvalues (spectral) decomposition is an alternative method, typically implemented by a variant of the power iterations algorithm, such as Lanczos iterations. Unfortunately, all existing spectral decomposition methods are often too slow to be applied to a NLM matrix. Low rank approximation of the NLM matrix via truncated spectral decomposition is given in Algorithm \ref{alg:nlm_eig} in Appendix~\ref{sec:apx_algorithms}.

Our approach is to construct a low-rank approximation of a NLM operator by using a matrix function ${f \colon \mathbb{R}^{n \times n}\rightarrow \mathbb{R}^{n \times n}}$, and computing $f\left(A\right)$, where $A$ is a NLM operator, as in Definition~\ref{def:nlmeans_op}. We next characterize functions of NLM operators and suggest properties of such functions which are suitable for our purposes. Then, we present a particular family of functions with controlled low-rank that have these properties.

\subsection{Matrix functions of NLM operators}

Before studying matrix functions on NLM operators, we summarize a few properties of these NLM matrices.
\begin{lemma}  \label{lemma:NLM_matrix_properties}
Let $A$ be a NLM operator, as defined in Definition~\ref{def:nlmeans_op}. Then $A$ has the following properties:
\begin{enumerate}
\item $A$ is positive diagonalizable, namely there exists an invertible matrix $Q$ that satisfies $A=Q^{-1} \Lambda Q$,
where $\Lambda$ is a diagonal matrix $\Lambda = \diag\left(\lambda_1,\ldots,\lambda_n\right)$, and $\lambda_1 \ge \lambda_2 \ge \cdots \ge \lambda_n \ge 0$.
\item $\lambda_1 = 1$ is the maximal eigenvalue of $A$ with a corresponding eigenvector $\boldsymbol{1}$ (the all-ones vector), that is, $A\boldsymbol{1}=\boldsymbol{1}$.
\end{enumerate}
\end{lemma}
The proof is given in Appendix \ref{sec:apn_proof_of_lemma}. As a first conclusion from Lemma \ref{lemma:NLM_matrix_properties} we have
\begin{corollary}  \label{cor:function_on_NLM}
In the notation of Lemma \ref{lemma:NLM_matrix_properties}, any matrix function ${f \colon \mathbb{R}^{n \times n}\rightarrow \mathbb{R}^{n \times n}}$ satisfies
\[f(A)=Q^{-1}f\left(\Lambda\right)Q .\]
\end{corollary}
The proof follows directly from the diagonalizability of $A$ and the definition of matrix functions, see e.g., ~\cite[Corollary 11.1.2]{golub2012matrix}.

We design $f$ so that $f(A)$ is a low-rank approximation of the NLM operator $A$. Moreover, we wish to guarantee that $f(A)$ retains the properties of $A$ listed in Lemma~\ref{lemma:NLM_matrix_properties}. The first property in Lemma~\ref{lemma:NLM_matrix_properties} follows directly from Corollary~\ref{cor:function_on_NLM} and allows for repeated applications of the operator $f(A)$, as the eigenvalues are non-negative and bounded by 1. The advantages of repeated applications of a denoising operator have been demonstrated in~\cite{singer2009diffusion}.  The second property in Lemma~\ref{lemma:NLM_matrix_properties}, that is $f(A)\boldsymbol{1}=\boldsymbol{1}$ (row stochastic normalization of $A$), is useful because then the elements of $y=f(A)x$ are affine combinations (linear combinations whose coefficients sum to one) of the elements of $x$. In particular it is desired for a denoising operator to map a constant signal to itself, see e.g., \cite{coifman1995translation}.

Summarizing the above discussion of desired properties of the matrix $f\left(A\right)$, we define the notion of an extended NLM operator.

\begin{definition}  \label{defn_extendedop}
A matrix $B \in \mathbb{R}^{n \times n}$ is an extended NLM operator if
$B$ is diagonalizable, such that $B\boldsymbol{1}=\boldsymbol{1}$, with all its eigenvalues in $[0,1]$.
\end{definition}
While it is possible to define an extended NLM operator such that its eigenvalues are in $(-1,1]$, we prefer for simplicity to use Definition~\ref{defn_extendedop} above, as this allows for a richer family of functions to be used in Section~\ref{subsec:construct_low_rank} below. This is consistent with the Gaussian kernel used in Definition~\ref{def:nlmeans_op}, which ensures that the resulting operator is non-negative definite.

Equipped with the latter we have,
\begin{theorem}  \label{theorem1}
Let $A$ be a NLM operator. The following conditions on $f$ are sufficient for $f\left(A\right)$ to be
an extended NLM operator.
\begin{enumerate}
\item \label{cond1} $f\left(1\right)=1$ .
\item \label{cond2} $f$ is defined on the interval $[0,1]$ and $\left\{ f\left(x\right) \mid 0\leq x\leq1\right\} \subseteq[0,1]$.
\end{enumerate}
\end{theorem}
\begin{proof}
By Corollary \ref{cor:function_on_NLM}, $f$ acts solely on the diagonal matrix $\Lambda$, meaning that
\[ f(A) = Q^{-1} \diag\left(f(\lambda_1),\ldots,f(\lambda_n)\right) Q , \]
where the eigenvalues of $A$ are ordered such that $\lambda_{1}=1$.
Therefore, the second condition guarantees that the eigenvalues of $f(A)$ are in $[0,1]$. In addition, the first condition ensures the maximal eigenvalue is indeed one, and is given by $f(\lambda_1)$, that is $f(\lambda_{1})=1$. The eigenvector corresponding to $f(\lambda_1)$ is the first column of $Q$, which is $\boldsymbol{1}$ (up to a scalar scale), namely, $f(A) \boldsymbol{1} = \boldsymbol{1}$.
\end{proof}

The following corollary allows composition of functions satisfying the conditions of Theorem~\ref{theorem1}.
\begin{corollary} \label{cor:function_composition}
Let $B$ be an extended NLM operator, and let $g$ be a function satisfying the conditions of Theorem~\ref{theorem1}. Then, $g\left(B\right)$ is also an extended NLM operator.
\end{corollary}
The proof is elementary and thus omitted.

\subsection{Constructing a low-rank extended NLM operator} \label{subsec:construct_low_rank}

We would like to construct a function $f(A)$ that satisfies the two conditions from Theorem~\ref{theorem1}, and moreover, suppresses the noise in the NLM operator $A$. A natural option is to choose a function that retains the eigenvalues above a given threshold. By Corollary \ref{cor:function_on_NLM}, choosing such a function reduces to choosing an appropriate scalar function.

Our first prototype for a scalar thresholding function is
\begin{equation} \label{eqn:g_thresholding_func}
g_\omega\left(x\right)=\begin{cases}
0 & x<\omega , \\
x &  \omega \le x .
\end{cases}
\end{equation}
This function satisfies the conditions of Theorem \ref{theorem1} for $0 \le \omega \le 1$, while zeroing values lower than $\omega$ and acting as the identity for higher values. However, this function is not smooth which eventually results in its slow evaluation for matrices (a detailed discussion is given in Section~\ref{sec:approximation_bound}).

Inspired by the gain function of the Butterworth filter~\cite{butterworth1930theory} (also known as maximal flat filter)
\[
f^{b}_{\omega,d}(x)=\left(1+\left(\frac{1-x}{1-\omega}\right)^{2d}\right)^{-\frac{1}{2}}, \quad x \in [0,1],
\]
which approximates the $\omega$-shifted Heaveside function on the segment $[0,1]$, we propose to use
\begin{equation}  f^{sb}_{\omega,d}(x)=x\left(1+\left(\frac{1-x}{1-\omega}\right)^{2d}\right)^{-\frac{1}{2}} , \quad x \in [0,1].
\label{eqn:sbwx}
\end{equation}
We term this function the Slanted Butterworth (SB) function. The SB function serves as an approximation to $g_\omega$ of \eqref{eqn:g_thresholding_func}. The response of both functions $f^{b}_{\omega,d}(x)$ and $f^{sb}_{\omega,d}(x)$ for various values of $d$ is shown in Figure~\ref{fig:filter_response}.

\begin{figure}
\centerline{
\begin{tabular}{cc}
\includegraphics[width=.45\textwidth]{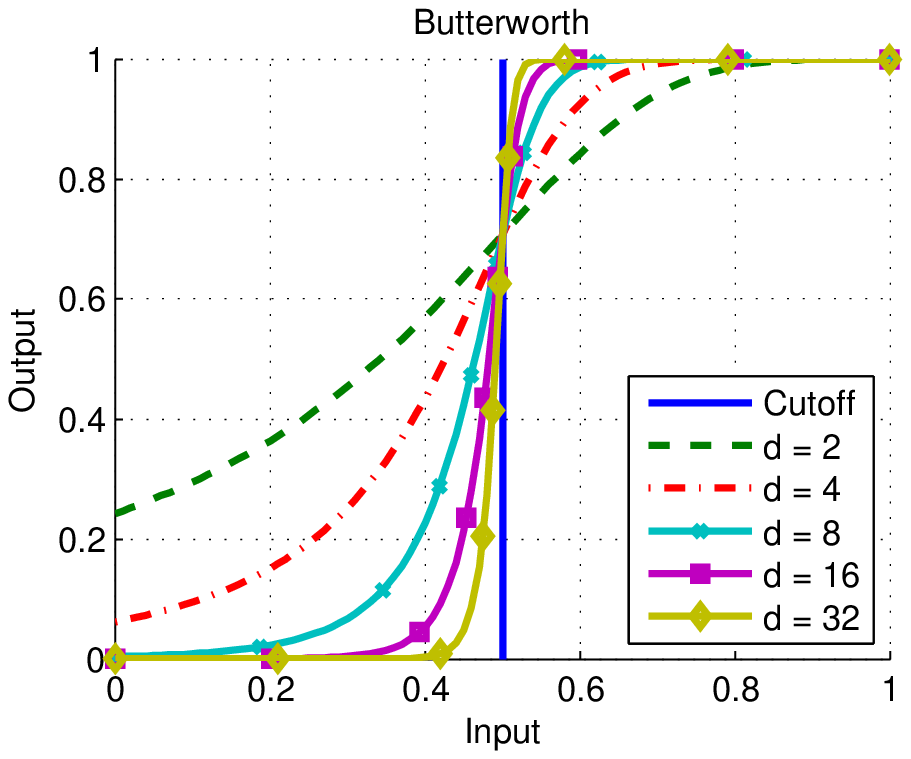} & \includegraphics[width=.45\textwidth]{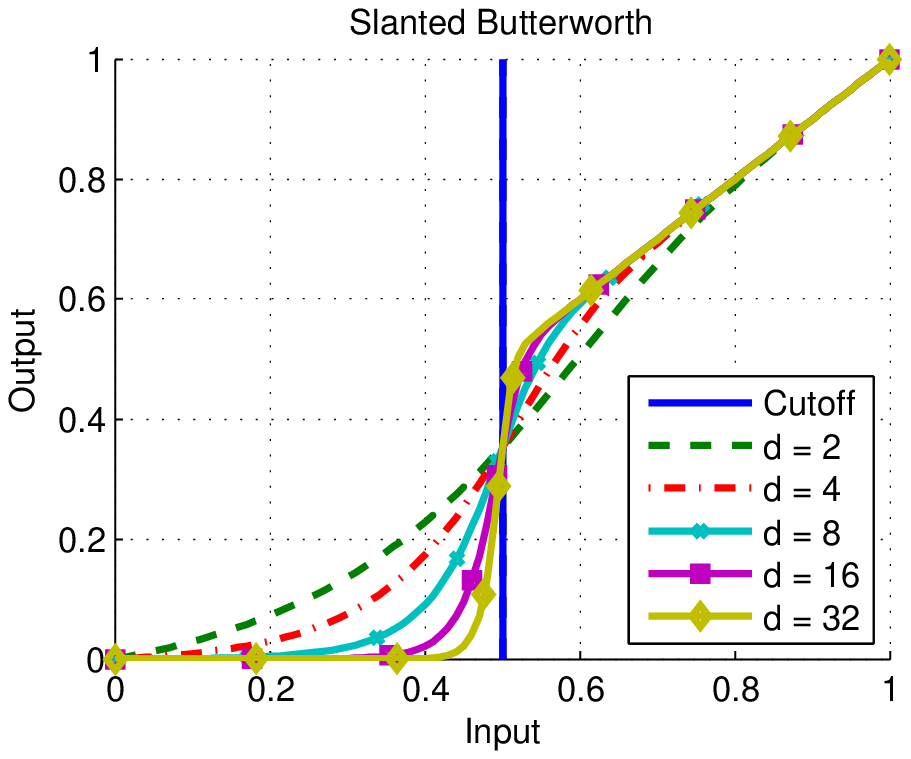}\tabularnewline
\end{tabular}}
\caption{Plots of the Butterworth and Slanted Butterworth functions, for cutoff set at 0.5 and different filter orders.}
\label{fig:filter_response}
\end{figure}

In its matrix version, the SB function becomes
\begin{equation}
f^{sb}_{\omega,d}\left(A\right)=A\left(I+\left(\frac{1}{1-\omega} (I-A) \right)^{2d}\right)^{-\frac{1}{2}},
\label{eqn:sbw}
\end{equation}
where $A\in\mathbb{R}^{n \times n}$ and $I$ is the $n\times n$ identity matrix. Combining Theorem \ref{theorem1} and Corollary \ref{cor:function_composition}, one can verify that $f^{sb}_{\omega,d}\left(A\right)$, $0 \le \omega \le 1$, is indeed an extended NLM operator and thus $f^{sb}_{\omega,d}$ is applicable for our purposes. The parameter $0\leq\omega\leq1$ is termed the filter cutoff and $d \in \mathbb{N}$ is the order of the filter.

\section{Computing low-rank approximation based on Chebyshev polynomials}\label{sec:computation}

The evaluation of the matrix function $f^{sb}_{\omega,d}(A)$ for a large matrix $A$ can be challenging. While the function essentially operates on the eigenvalues of the NLM operator $A$, the spectral decomposition of $A$ is assumed to be unavailable due to computational reasons.
Furthermore, evaluating the function $f^{sb}_{\omega,d}$ directly according to its definition~\eqref{eqn:sbw} requires computing the square root of a matrix. Evaluating the square root of a matrix is prohibitively computationally expensive, see e.g.~\cite{higham1997stable}, and thus, a direct evaluation of $f^{sb}_{\omega,d}$ is also infeasible. In addition, an important observation is that one does not need the resulting matrix $f^{sb}_{\omega,d}\left(A\right)$ but only the vector $\hat{x}=f^{sb}_{\omega,d}\left(A\right) y$, as can be seen from~\eqref{eq:A}.

\subsection{Evaluating the SB function based on Chebyshev polynomials}

The Chebyshev polynomials of the first kind of degree $n$ are defined as
\begin{equation*} 
 T_n(x) = \cos \left( n \arccos(x) \right) , \quad x \in [-1,1] , \quad n = 0,1,2,\ldots
\end{equation*}
These polynomials satisfy the three term recursion
\begin{equation} \label{eqn:three_term_recursion}
  T_n(x) = 2xT_{n-1} - T_{n-2} , \quad n=2,3,\ldots
\end{equation}
with $T_0(x)=1$ and $T_1(x) = x$, and form an orthogonal basis for  $L_2([-1,1])$ with the inner product
\begin{equation} \label{eqn:inner_product}
 \left\langle  f,g \right\rangle_T =  \frac{2}{\pi} \int_{-1}^1 \frac{f(t)g(t)}{\sqrt{1-t^2}} \, dt .
\end{equation}
The Chebyshev expansion for any $f \in L_2([-1,1])$ is thus given by
\begin{equation} \label{eqn:chebyshev_series_scalars}
f(x) = \sum_{j=0}^\infty \alpha_j T_j(x) , \quad  \alpha_0 = \frac{1}{2}\left\langle f, T_0 \right\rangle_T  , \quad  \alpha_n = \left\langle f, T_n \right\rangle_T , \quad n\in\mathbb{N} ,
\end{equation}
where the equality above holds in the $L_{2}$ sense for any $f\in L_{2}([-1,1])$, and becomes pointwise equality under additional regularity assumptions on $f$.

We will evaluate the function $f^{sb}_{\omega,d}:[0,1]\to\mathbb{R}$ by truncating the Chebyshev expansion \eqref{eqn:chebyshev_series_scalars}, that is
\begin{equation*} 
f^{sb}_{\omega,d}\left(z\right)\approx \sum_{j=0}^{N-1} \alpha_{j}T_{j}\left(y\right),
\end{equation*}
where $\alpha_j$ are the corresponding Chebychev coefficients for $f^{sb}_{\omega,d}$, and $y=2z-1$ maps $f^{sb}_{\omega,d}$ from $[0,1]$ to $[-1,1]$, as required by the definition of $\alpha_j$ above.

As shown in Corollary~\ref{cor:function_on_NLM}, applying a Chebychev expansion to a NLM matrix $A$ is reduced to applying it to the diagonal form of $A$. Thus,
\begin{equation} \label{eqn:chebyshev_element_wise}
\begin{aligned}
f^{sb}_{\omega,d}\left(A\right) &= Q^{-1} f^{sb}_{\omega,d}(\Lambda) Q  \\
&= Q^{-1} \left( \sum_{j=0}^\infty\alpha_{j}T_{j}(2\Lambda-I ) \right) Q \\
&= Q^{-1}  \left( \sum_{j=0}^\infty \diag ( \alpha_{j}T_{j} (2\lambda1-1),\ldots ,\alpha_{j}T_j(2\lambda_n-1) ) \right) Q
\end{aligned}
\end{equation}
where the second equality holds for any $A$ since $f^{sb}_{\omega,d}$ is a continuous function \cite[Chapter 8]{trefethen2000spectral}. In other words, one can see that the coefficients $\{\alpha_j\}$ required to evaluate the matrix function $f^{sb}_{\omega,d}\left(A\right)$ of~\eqref{eqn:sbw} are the same as those required to evaluate the scalar function $f^{sb}_{\omega,d}\left(x\right)$ of~\eqref{eqn:sbwx}.

For square matrices, substituting a matrix inside a polynomial is well-defined (see e.g. \cite[Chapter 11]{golub2012matrix}) and so given a truncation parameter $N \in \mathbb{N}$, the matrix SB function \eqref{eqn:sbw} is approximated by
\begin{equation}\label{eqn:cheb_mat_approx}
f^{sb}_{\omega,d}\left(A\right)\approx S_N(f^{sb}_{\omega,d},A) = \sum_{j=0}^{N} \alpha_{j}T_{j}\left(2A-I\right) .
\end{equation}

The common practice for calculating Chebyshev expansions is by using a Gauss quadrature formula for \eqref{eqn:inner_product} combined with the discrete orthogonality of the Chebyshev polynomials. Explicitly, the coefficients $\alpha_{j}$ are calculated by \eqref{eqn:chebyshev_series_scalars} combined with
\begin{equation} \label{eqn:Cheby_Coef}
\left\langle f, T_j \right\rangle_T =\frac{1}{N+1}\sum_{k=1}^{N+1} f(x_k)T_j(x_k) , \quad x_k=\cos \left( \frac{\pi(k - \frac{1}{2})}{N}\right) , \quad j=0,\ldots,N .
\end{equation}
For more details see~\cite[Chapter 5.8]{flannery1992numerical}.

Having obtained the Chebyshev expansion coefficients $\left\lbrace \alpha_j \right\rbrace_{j=0}^{N}$ of \eqref{eqn:cheb_mat_approx}, we turn to show how to efficiently evaluate $\hat{x}=S_N(f^{sb}_{\omega,d},A) y$. We do the latter by using a variant of Clenshaw's algorithm (see e.g., \cite[Chapter 5.5]{flannery1992numerical}), which is based on the three term recursion \eqref{eqn:three_term_recursion}, as described in \cite[p. 193]{flannery1992numerical}. This algorithm is adapted to matrices using the fact that each polynomial consists only of powers of $A$, which means that any two matrix polynomials commute. In addition, we exploit the fact that we need to compute only the vector $\hat{x}$ and not the matrix $S_N(f^{sb}_{\omega,d},A)$ itself, which has much larger dimensions. The pseudocode for evaluating $S_N(f^{sb}_{\omega,d},A)y$ for some vector $y$ is given in Algorithm~\ref{alg:clenshaw} in Appendix~\ref{sec:apx_algorithms}. Note that this algorithm does not require any matrix-matrix operations, and thus feasible even for large matrix sizes.

The complete denoising algorithm, which computes a denoised image based on $S_N(f^{sb}_{\omega,d},A)$ for a NLM operator $A$, is presented in Algoritm~\ref{alg:nlm_sbw} in Appendix \ref{sec:apx_algorithms}.

\subsection{Error bound for NLM operators}\label{sec:approximation_bound}  \label{sec:app_error}

We study the approximation error of the truncated Chebychev expansion for the case of matrix functions and for NLM operators, which are diagonalizable, non-symmetric matrices. The use of truncated Chebyshev expansions for matrices \eqref{eqn:cheb_mat_approx} and their approximation order has already been studied in the context of solutions for partial differential equations \cite{tal1989polynomial, trefethen2000spectral}. However, most results
assume that the approximated function is analytic, and so not applicable in our case.

In this section we use the following notation. Denote by $\| X \|_F = \sqrt{ \operatorname{tr}(XX^T) } $ the Frobenius norm of $X$ and by $\norm{X} $ its spectral norm (or the induced $2$-norm), that is the largest singular value of $X$. In addition, denote by $\kappa(X)$ the condition number of an invertible matrix $X$ (with respect to the spectral norm), that is $\kappa(X)=\| X^{-1} \| \| X \|$.

Recall that if $A$ is a NLM operator, then $A$ can be decomposed as $A=Q^{-1} \Lambda Q$ (see Lemma~\ref{lemma:NLM_matrix_properties}). In addition, $A$ is conjugated to a symmetric matrix via $D^{-\frac{1}{2}}$ with $D$ being the diagonal matrix of Definition~\ref{def:nlmeans_op}. Therefore, $Q = D^{-\frac{1}{2}} O$, where $O$ is an orthogonal matrix.

The next theorem presents the main error bound.
\begin{theorem} \label{thm:truncated_cheby}
Let $f \in C^{m+1}([-1,1])$ and let $A$ be an $n \times n$ NLM operator. Denote by
\[ e_N(f)(x) = f(x) -S_N(f,x) \]
the approximation error of the truncated Chebychev expansion of degree $N$. Then,
\[ \| e_N(f)(A) \| \leq C \frac{1}{(N - m)^m} \kappa(D^{1/2}), \]
where $C = \frac{2}{\pi m} \|f^{(m+1)}\|_T  $ is a constant that depends on the $m+1$ derivative of $f$ but independent of $n$ and $A$, with
\[ \| f^{(m+1)} \|_T = \int_{-1}^1 \frac{ \abs{f^{(m+1)}(t)} }{\sqrt{1-t^2}} \, dt . \]
\end{theorem}

\begin{proof}
Since $A$ is diagonalizable, and similarly to Corollary \ref{cor:function_on_NLM} and \eqref{eqn:chebyshev_element_wise},
\begin{equation*}
e_N(f)(A)=f(A)-\sum_{j=0}^N \alpha_jT_j(A)=Q^{-1}(f(\Lambda)-\sum_{j=0}^N \alpha_jT_j(\Lambda))Q=Q^{-1}E_\Lambda Q,
\end{equation*}
where $E_\Lambda$ is the diagonal matrix $E_\Lambda=f(\Lambda)-\sum_{j=0}^N \alpha_jT_j(\Lambda)$. For all submultiplicative norms, including the spectral norm, we have that
\begin{equation}\label{eq2}
\| e_N(f)(A) \| \leq \|Q^{-1}\| \| E_\Lambda \| \| Q \|.
\end{equation}
By the Chebychev approximation bound \cite[Theorem 4.3]{trefethen2008gauss} for scalar functions,
\[ \norm{ \left(E_\Lambda\right)_{ii} } \leq \frac{C}{(N - m)^m} , \quad   1 \le i \le n , \]
where $ C = \frac{2}{\pi m}\|f^{(m+1)}\|_T $. The spectral norm of a diagonal matrix is given by its element on the diagonal having maximal absolute value and thus
 \begin{equation}\label{eq3}
\norm{ E_\Lambda } \leq \frac{C}{(N - m)^m}.
\end{equation}
Now, it remains to find a bound on $\|Q^{-1}\| \| Q \|$. However,
\begin{equation}
\|Q^{-1}\| \| Q \| = \| O^{*} D^{1/2} \| \| D^{-1/2} O \|.
\end{equation}
Using again the submultiplicativity property we get
\begin{equation}\label{eq6}
\| Q^{-1} \|  \| Q \| \leq \| O \| \| O^* \| \| D^{1/2} \| \| D^{-1/2} \|=  \kappa(O)\kappa(D^{1/2}),
\end{equation}
where in the last equality we have used the orthogonality of $O$. By the same property, we have that $\| O \|=1$ and $\kappa(O)=1$. Combining the latter with \eqref{eq2}, \eqref{eq3}, and \eqref{eq6} concludes the proof.
\end{proof}

The proof of Theorem \ref{thm:truncated_cheby} holds with minor adjustments to various submultiplicative norms. However, for one particular norm a good bound can be achieved directly, as explained in the following remark.
\begin{remark} \label{rmk:Fro_remark}
The error bound of the truncated Chebyshev expansion for matrices, as appears in Theorem \ref{thm:truncated_cheby}, can be also expressed in terms of the Frobenius norm since $\| X \|_F \leq \sqrt{n} \| X \|$. Therefore, we immediately get
\[ \| e_N(f)(A) \|_F \leq C \frac{\sqrt{n}}{(N - m)^m} \kappa_F(D^{1/2}), \]
where $\kappa_F(D^{1/2}) = \norm{D^{1/2}}_F \norm{D^{-1/2}}_F $.
\end{remark}

The bound of Theorem \ref{thm:truncated_cheby} indicates that the approximation error decays as $\frac{1}{N^m}$, where $m$ is related to the smoothness class of the function $f$, and $N$ is the number of terms in the truncated Chebyshev expansion \eqref{eqn:cheb_mat_approx}. However, there are two additional factors that appear on top of this decay rate. The first one is the constant $C$, governed by $\norm{f^{(m+1)}}_T$. This constant can be large for a function whose $m+1$ derivative has large magnitude. The second one is the condition number of $D^{1/2}$. This condition number can be easily calculated numerically since $D$ is a diagonal matrix and thus
\begin{equation} \label{eqn:kappa_spectral}
 \kappa(D^{1/2}) = \frac{\max_i \sqrt{D_{ii}} }{\min_i \sqrt{D_{ii}}} .
\end{equation}
Similarly, for the Frobenius norm we have
\begin{equation} \label{eqn:kappa_fro}
 \kappa_F(D^{1/2}) = \sqrt{\frac{\sum_i D_{ii} }{\sum_i D^{-1}_{ii} }  }  .
\end{equation}
Moreover, we can bound \eqref{eqn:kappa_spectral} and \eqref{eqn:kappa_fro} a priori, as seen next.

\begin{lemma} \label{lemma:D_W_bound}
Let $D$ be the diagonal matrix from Definition \ref{def:nlmeans_op}. Then,
\[  \kappa(D^{1/2}) \leq \sqrt{n}  \qquad \text {and} \qquad  \kappa_F(D^{1/2}) \leq n .\]
\end{lemma}

\begin{proof}
$K_h$ in Definition~\ref{def:nlmeans_op} is a Gaussian kernel and so $W_{ii}=1$ and $0 \leq W_{ij} \leq 1$. Thus, $ 1 \leq D_{ii} \leq n$ and $ 1 \leq D^{1/2}_{ii} \leq \sqrt{n}$. Therefore, by \eqref{eqn:kappa_spectral} we have $\kappa_2(D^{1/2}) \leq \sqrt{n}$.

For the Frobenius norm, it follows that $ \frac{1}{n} \leq D^{-1}_{ii} \leq 1$ and thus $\sum_i D_{ii} \le n^2$ and $\sum_i D^{-1}_{ii} \ge 1$. Thus, by \eqref{eqn:kappa_fro}, $ \kappa_F(D^{1/2}) \leq n$.
\end{proof}

Equipped with Lemma~\ref{lemma:D_W_bound}, we conclude that
\begin{corollary}
In the notation of Theorem~\ref{thm:truncated_cheby} we have
\[ \| e_N(f)(A) \| \leq C \frac{ \sqrt{n}}{(N - m)^m}. \]
\end{corollary}

Note that by Remark \ref{rmk:Fro_remark}, a bound on the Frobenius norm of the approximation error $e_N(f)(A)$ is given by
\[ \| e_N(f)(A) \|_F \leq C \frac{n^{1.5}}{(N - m)^m} . \]

To conclude the above discussion, we evaluate the approximation error numerically, depicted in Figure~\ref{fig:approx_accuracy}, where we use $50$ random, positive diagonalizable, and non-symmetric matrices, whose rank equals to $1000$. We average the relative approximation errors defined as $ \norm{E_N(f^{sw}_{\omega,d},A)} \Big/ \norm{f^{sw}_{\omega,d}(A)}$ over all 50 matrices, where the truncated Chebyshev series is evaluated by Algorithm \ref{alg:clenshaw}. The cutoff parameter $\omega$ of $f^{sw}_{\omega,d}$ has been set to $\omega=0.7$. The plotted curves correspond to the orders $4$, $8$ and $16$ of the function $f^{sw}_{\omega,d}$. The results clearly show that as the derivative grows, the error rate increases, since as $d$ gets larger so are the derivatives of $f^{sw}_{\omega,d}$.

\begin{figure}
\centerline{\includegraphics[width=.5\textwidth]{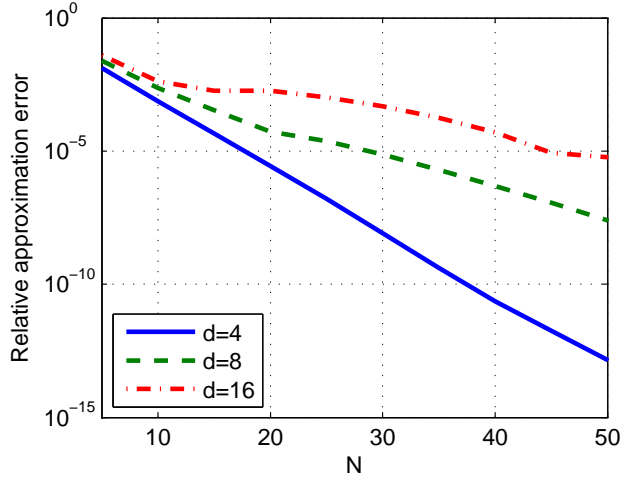}}
\caption{Relative approximation errors for the truncated Chebyshev expansion of $f^{sw}_{\omega,d}$ with a fixed $\omega$  and different values of $d$.}
\label{fig:approx_accuracy}
\end{figure}

\section{Numerical experiments}\label{sec:experiments}

The advantage of improving the NLM operator by using its low-rank approximation has already been argued in Section \ref{sec:spectrum shaping}. In this section we demonstrate this advantage by numerical examples, and in addition, study the effect of the two main parameters of our method -- the kernel width $h$ from Definition~\ref{def:nlmeans_op} and the filter cutoff $\omega$ from~\eqref{eqn:sbw}. As a reference, we use the naive approach for denoising a NLM operator, by computing its truncated eigenvalues decomposition (Algorithm~\ref{alg:nlm_eig}).

The numerical experiments are performed
on real images corrupted by synthetically added white Gaussian noise (as described in Section \ref{sec:nlmeans}) of varying levels. The clean test images of size $120 \times 120$ pixels are shown in Figure~\ref{fig:clean_images}.

\begin{figure}

\centering%
\subfloat[Mandril]{
\includegraphics[height=0.14\textheight]{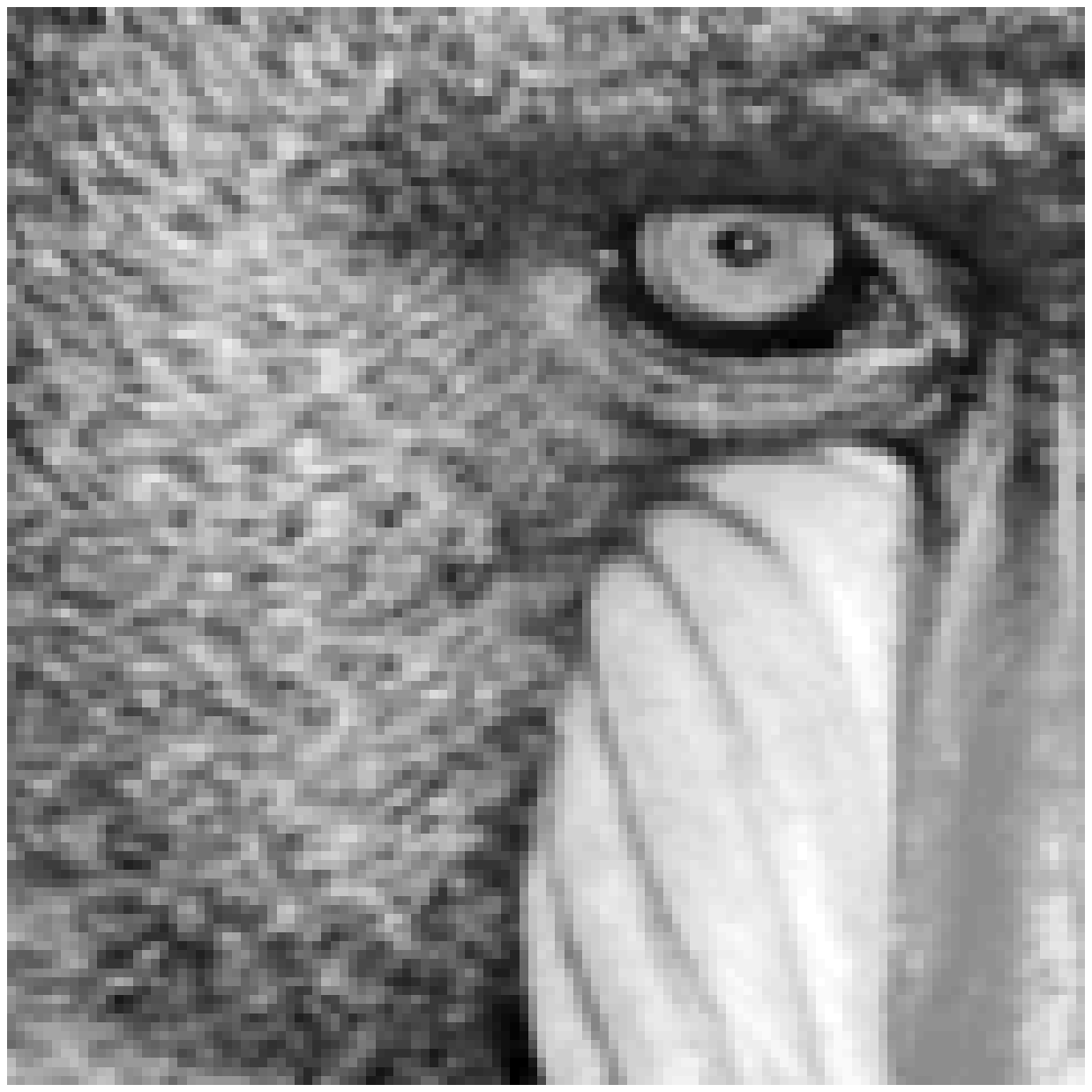}
}%
\subfloat[Barbara]{
\includegraphics[height=0.14\textheight]{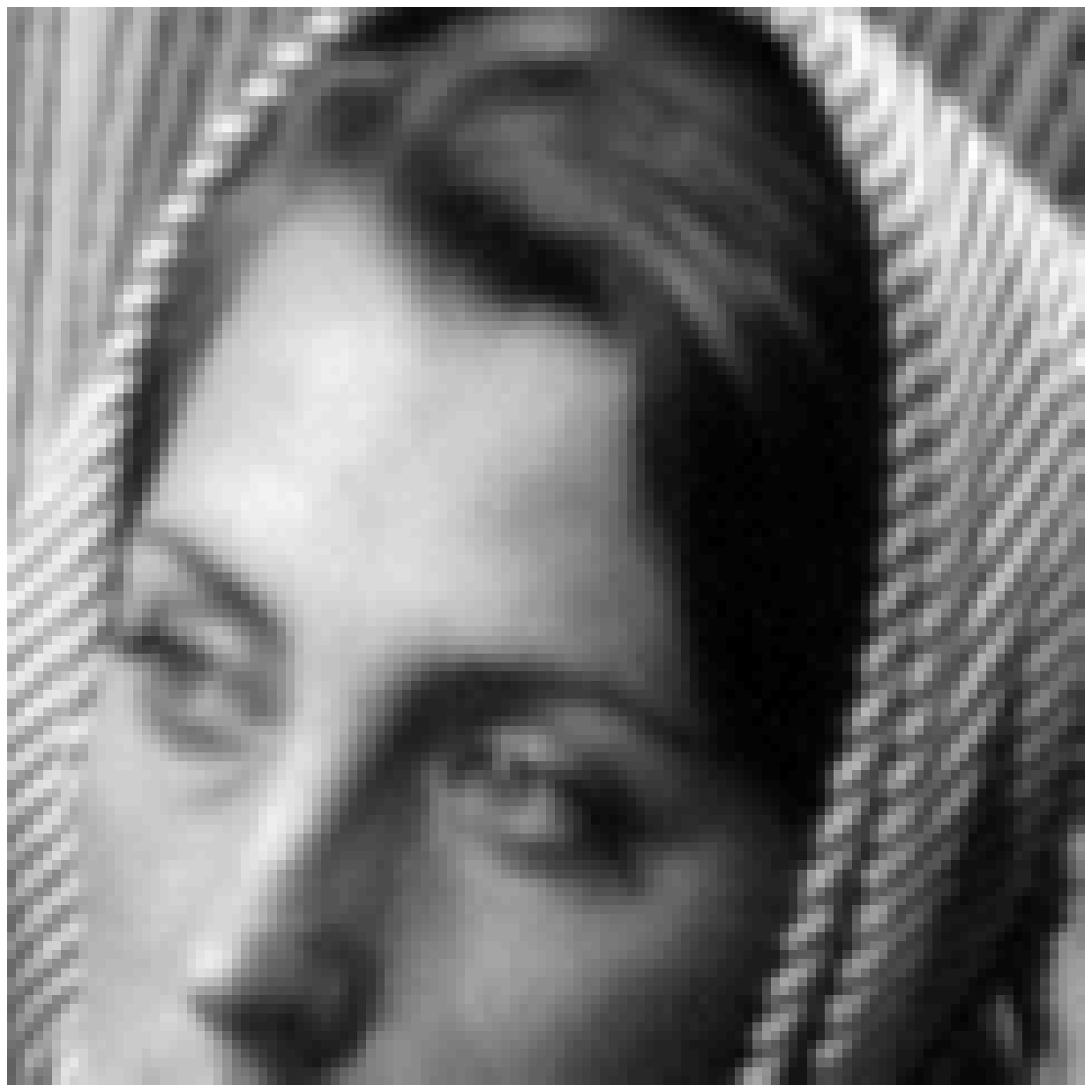}
}%
\subfloat[Roof]{
\includegraphics[height=0.14\textheight]{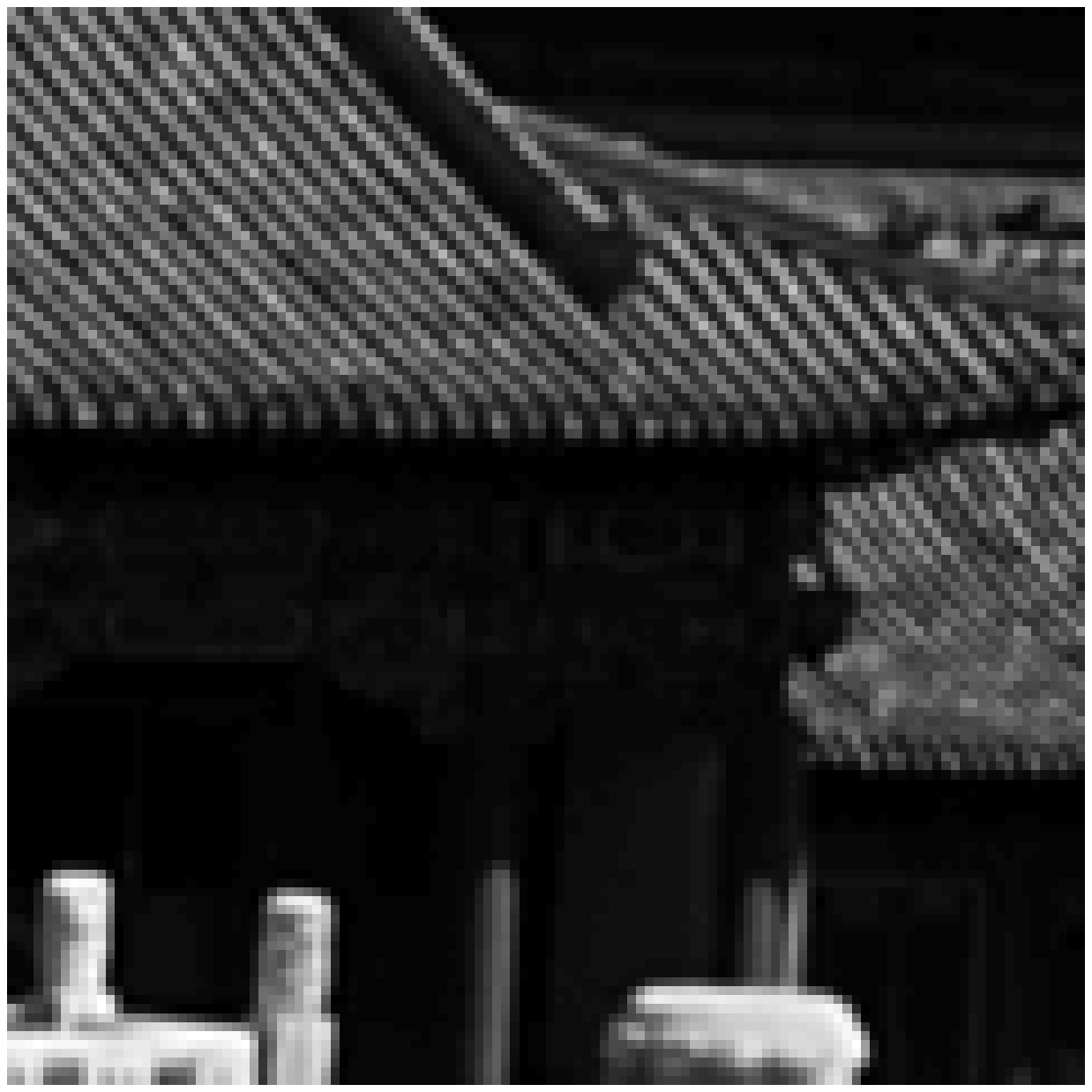}
}\\
\subfloat[Lena]{
\includegraphics[height=0.14\textheight]{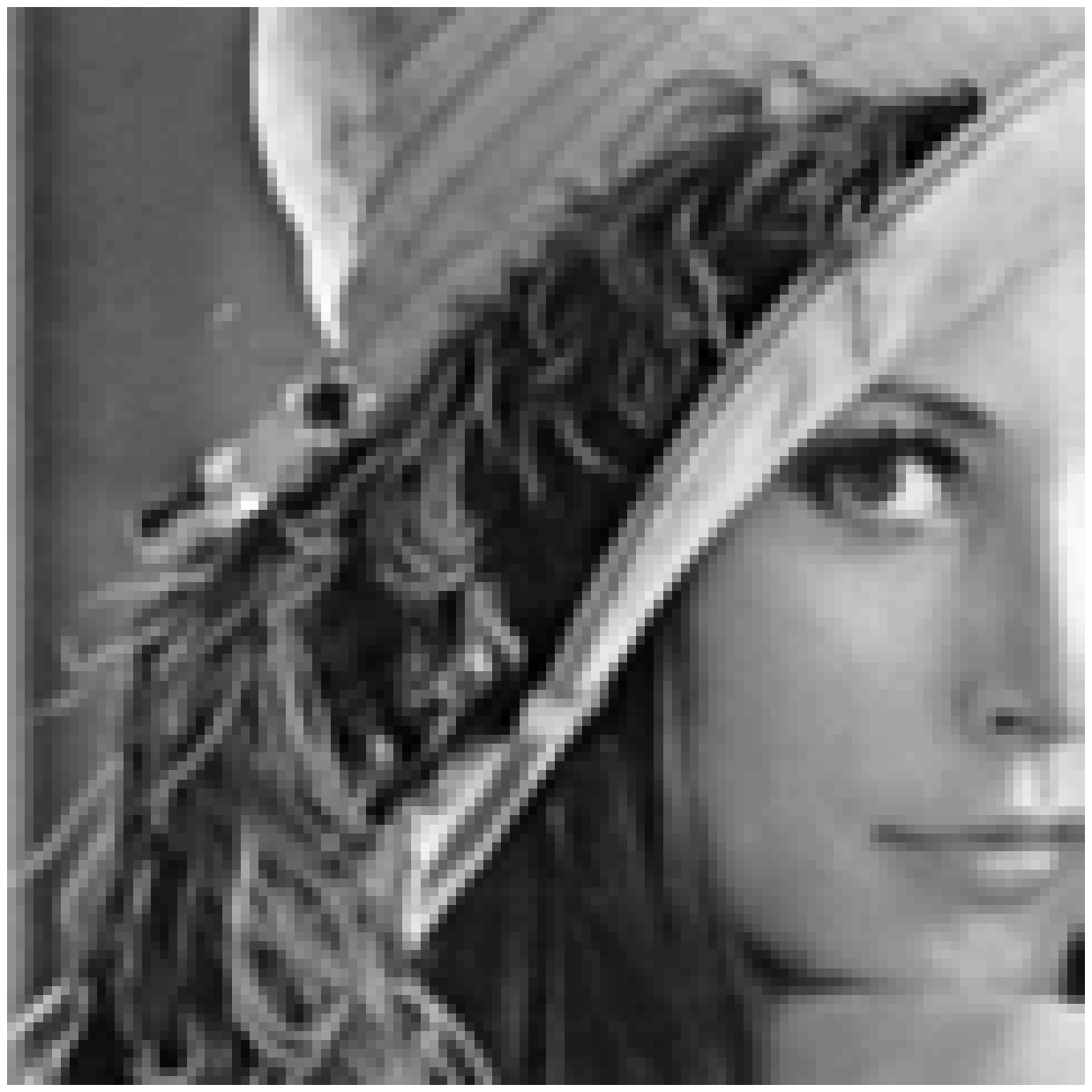}
}%
\subfloat[Clown]{
\includegraphics[height=0.14\textheight]{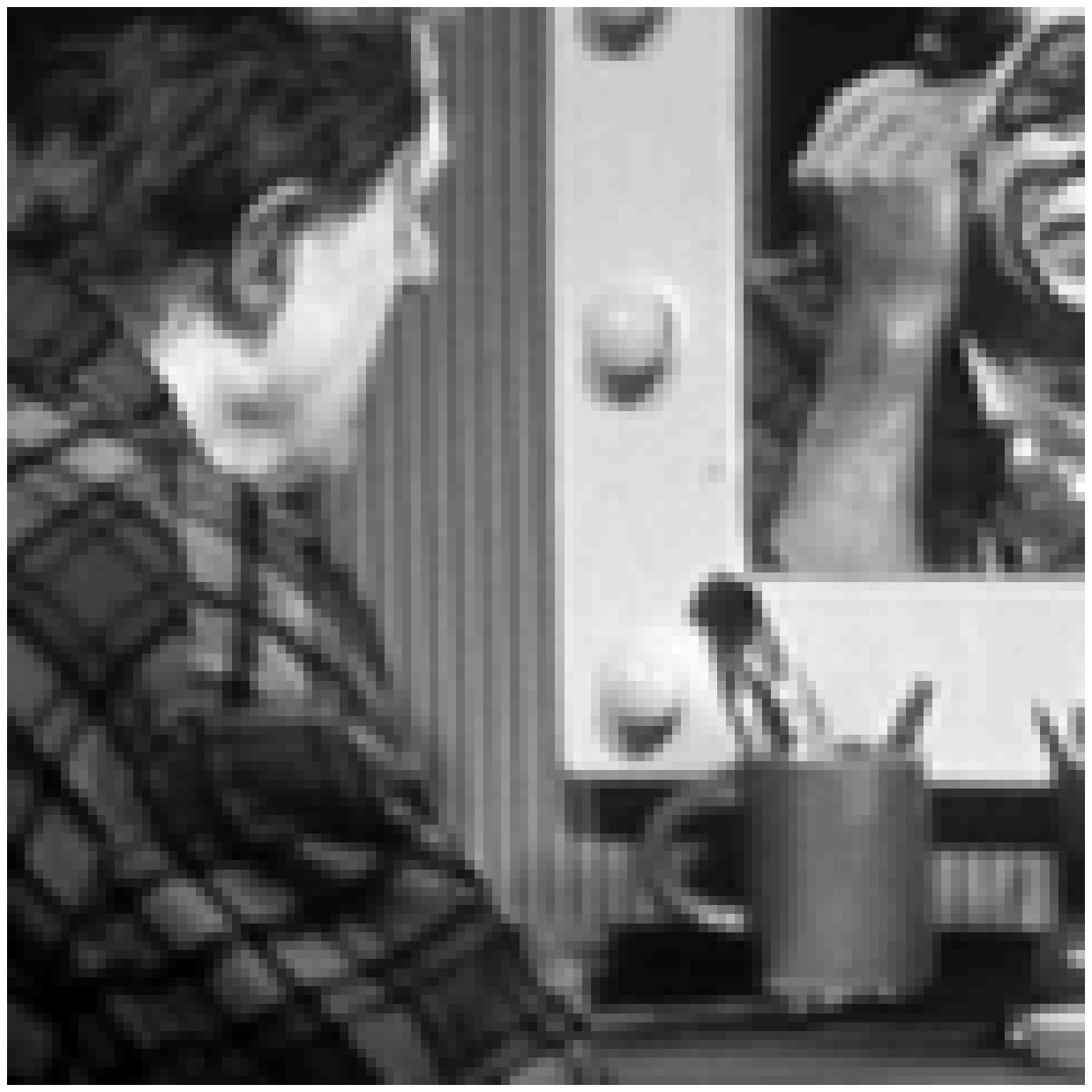}
}%
\subfloat[Couple]{
\includegraphics[height=0.14\textheight]{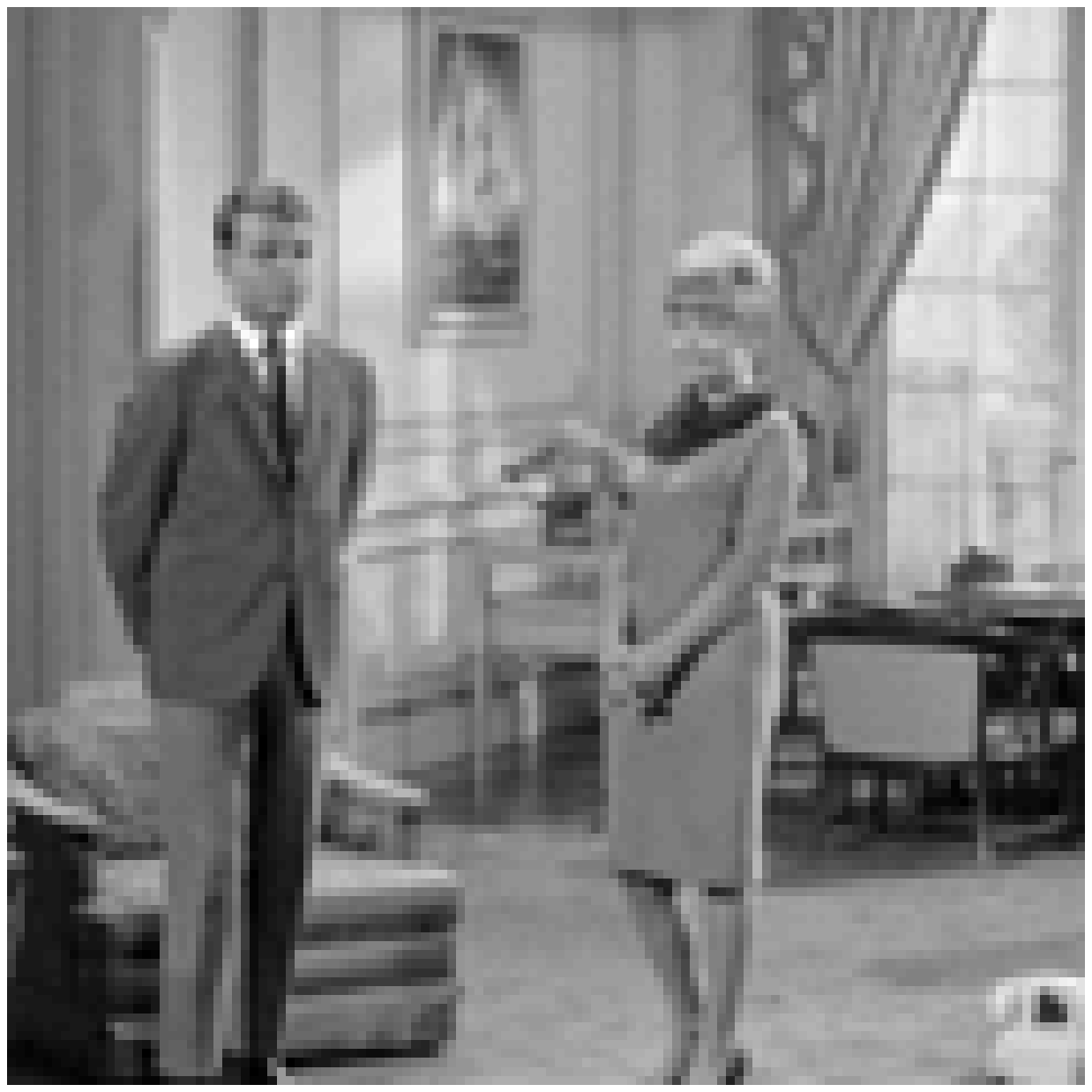}
}\\
\subfloat[Pentagon]{
\includegraphics[height=0.14\textheight]{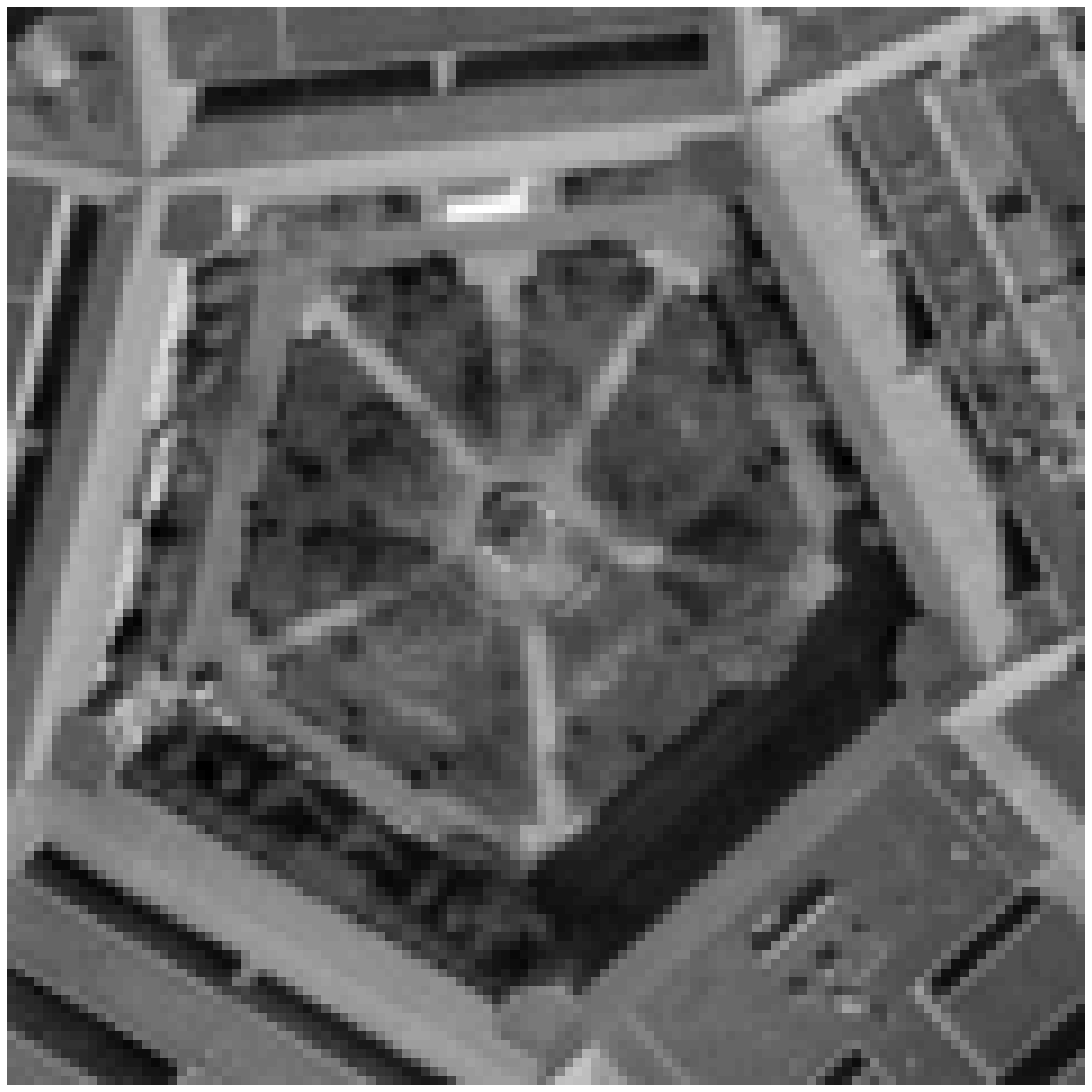}
}%
\subfloat[Woman\_Blonde]{
\includegraphics[height=0.14\textheight]{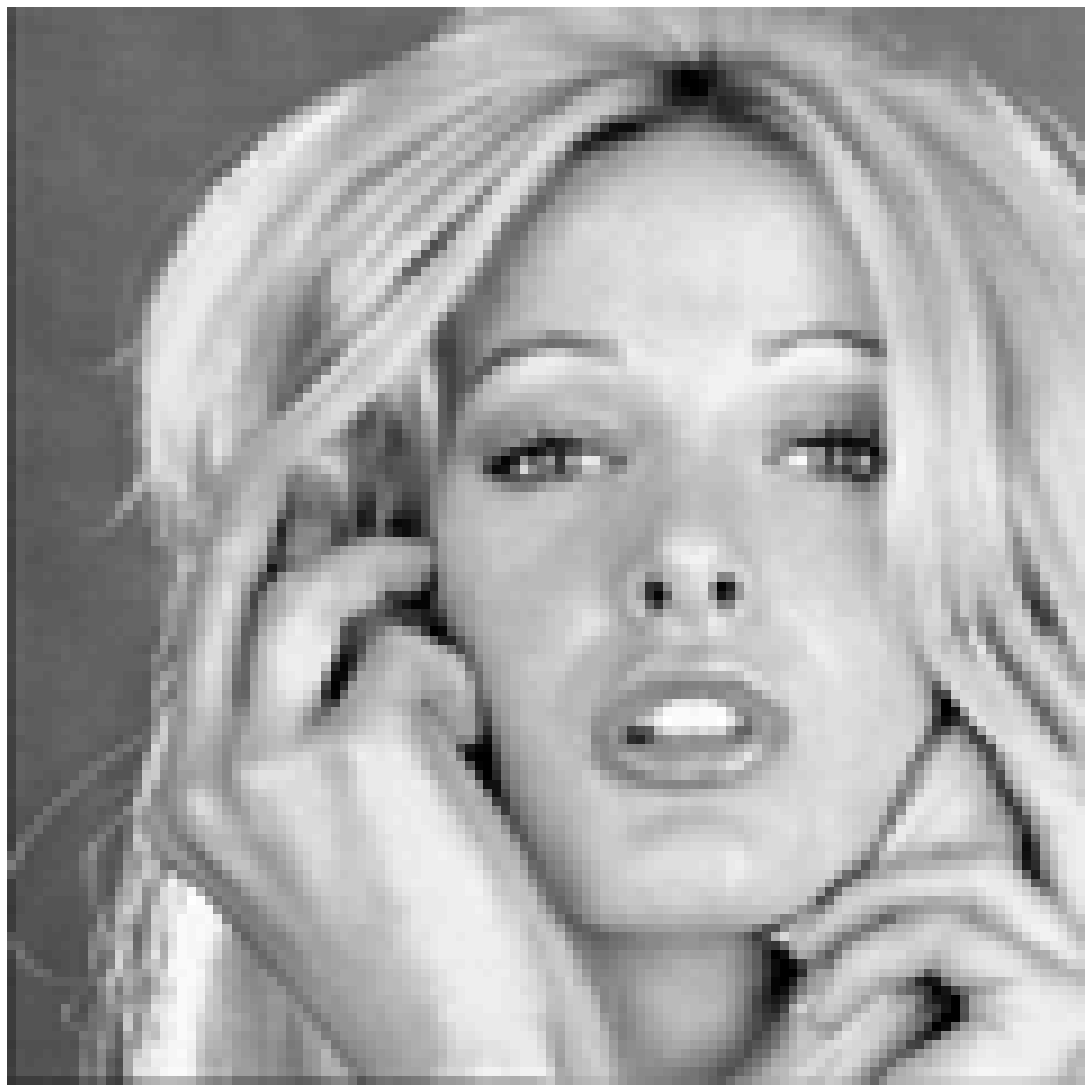}
}%
\subfloat[Man]{
\includegraphics[height=0.14\textheight]{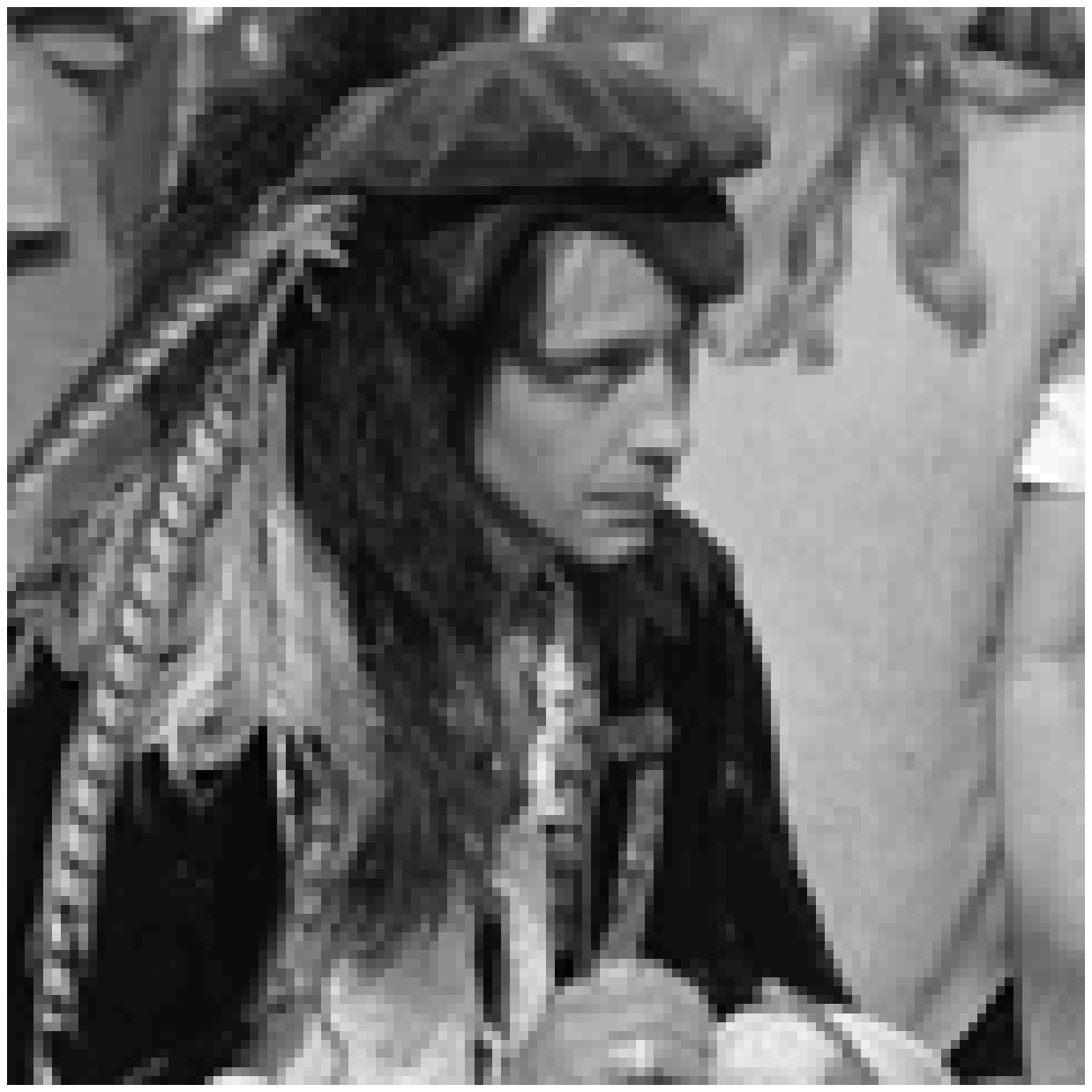}
}\\
\subfloat[Lake]{
\includegraphics[height=0.14\textheight]{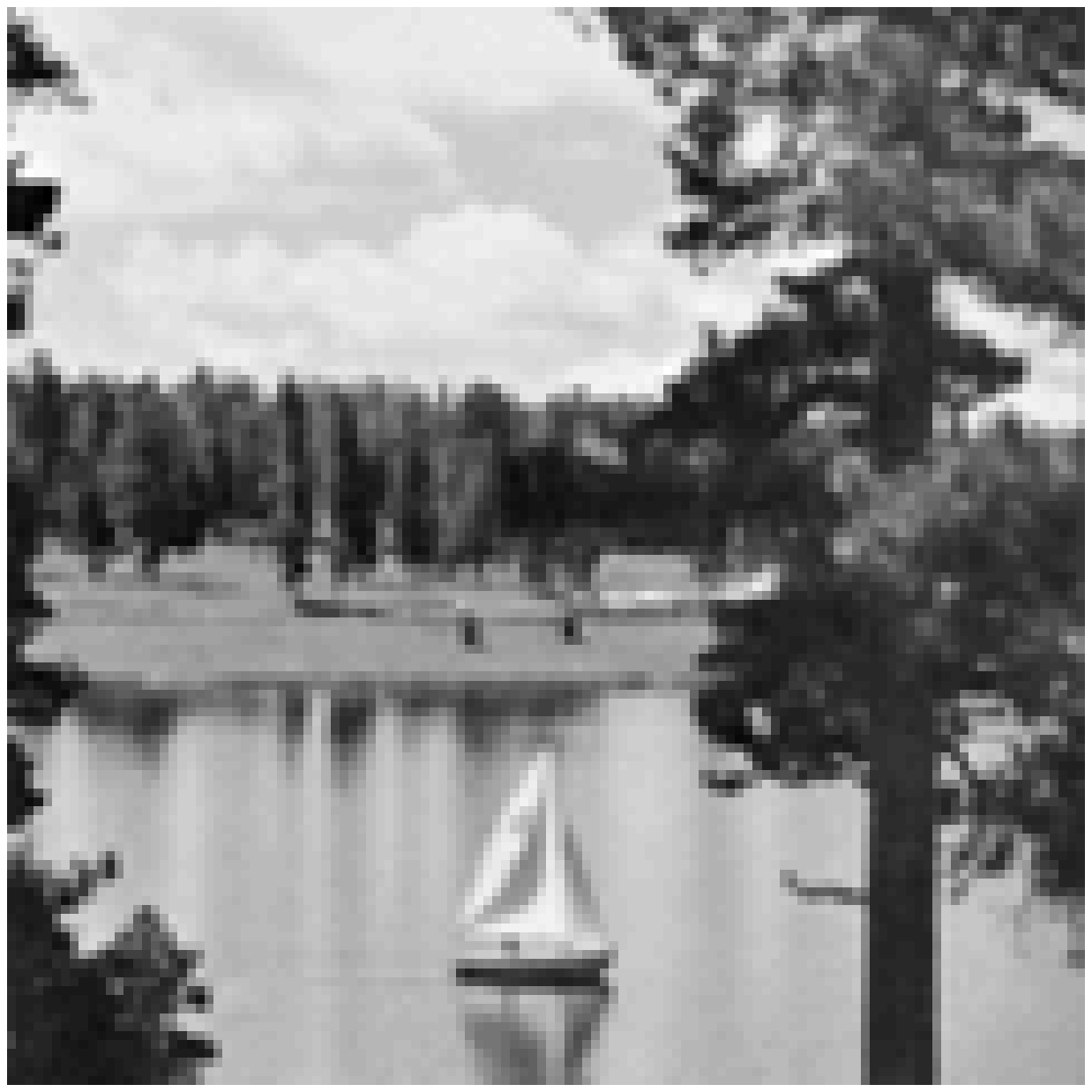} }%
\subfloat[Hill]{
\includegraphics[height=0.14\textheight]{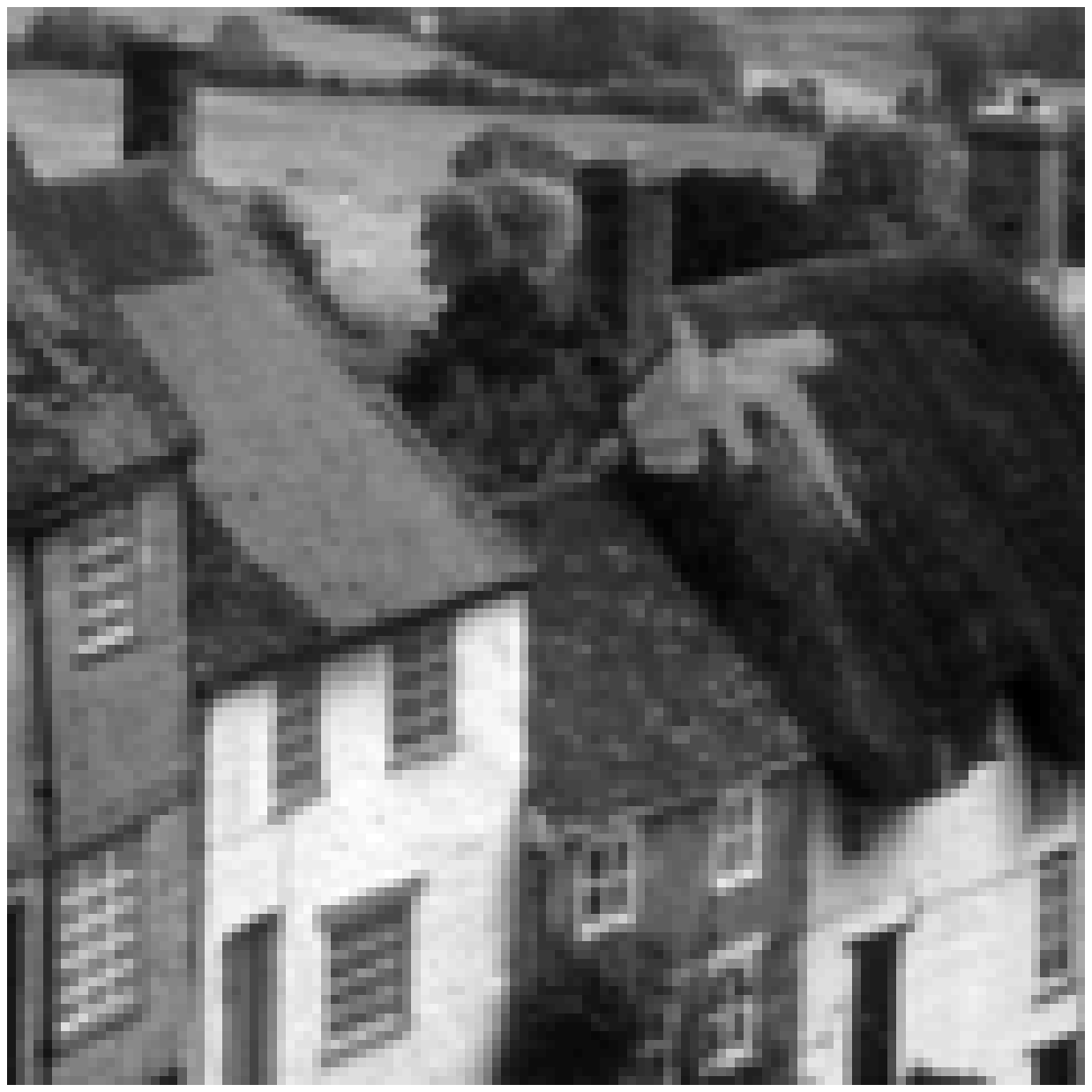}
}%
\subfloat[House]{
\includegraphics[height=0.14\textheight]{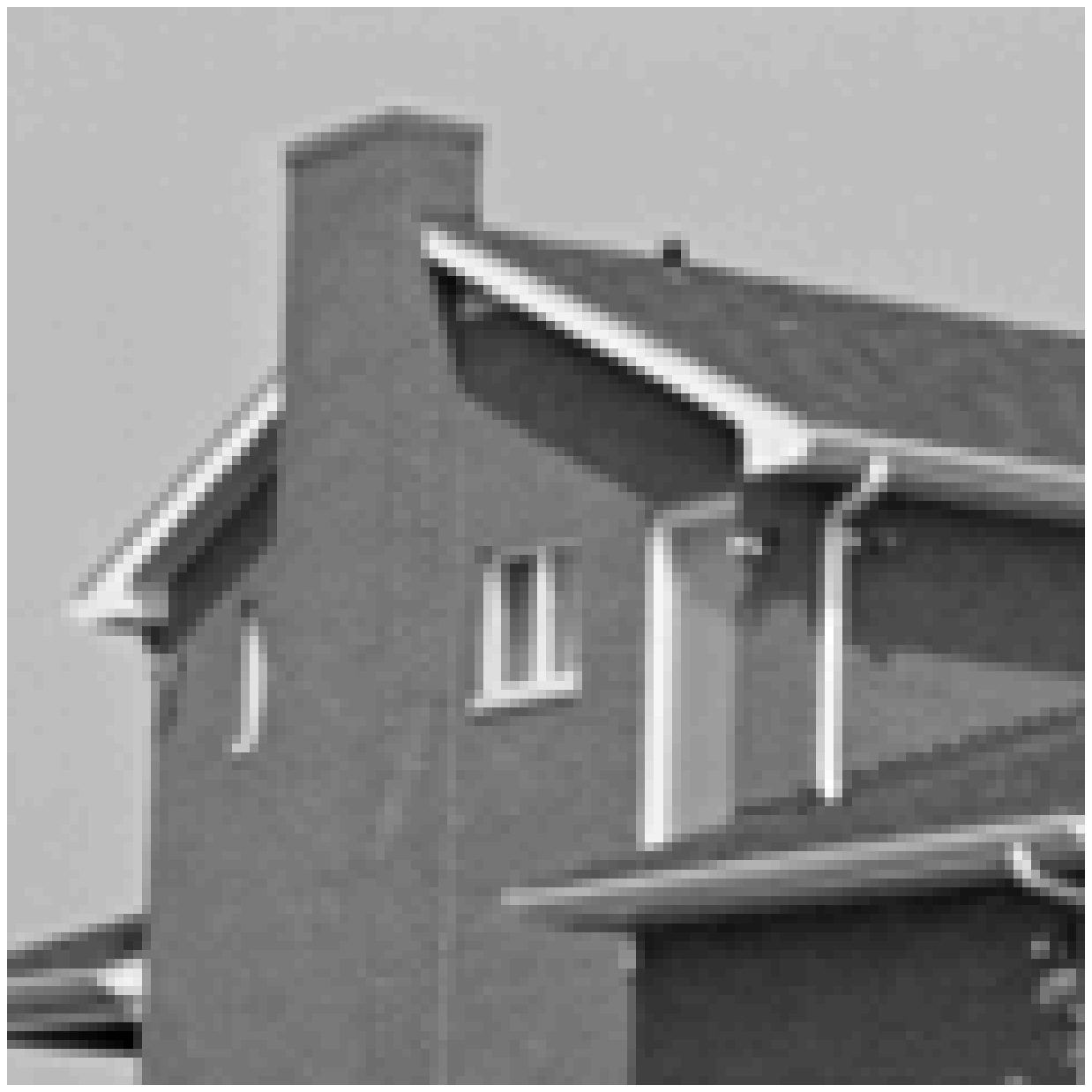}
}\\
\subfloat[Boat]{
\includegraphics[height=0.14\textheight]{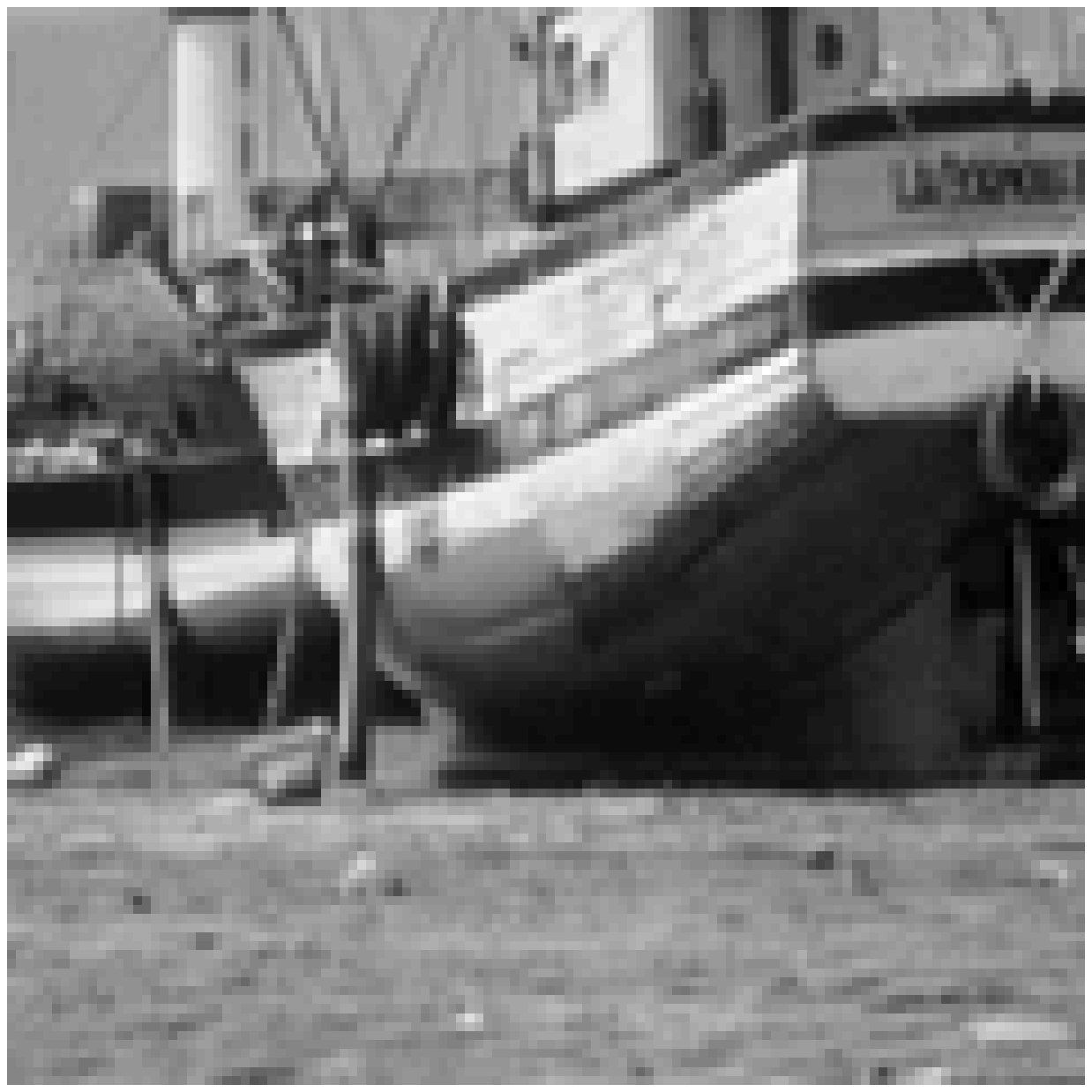}
}%
\subfloat[Man\_2]{
\includegraphics[height=0.14\textheight]{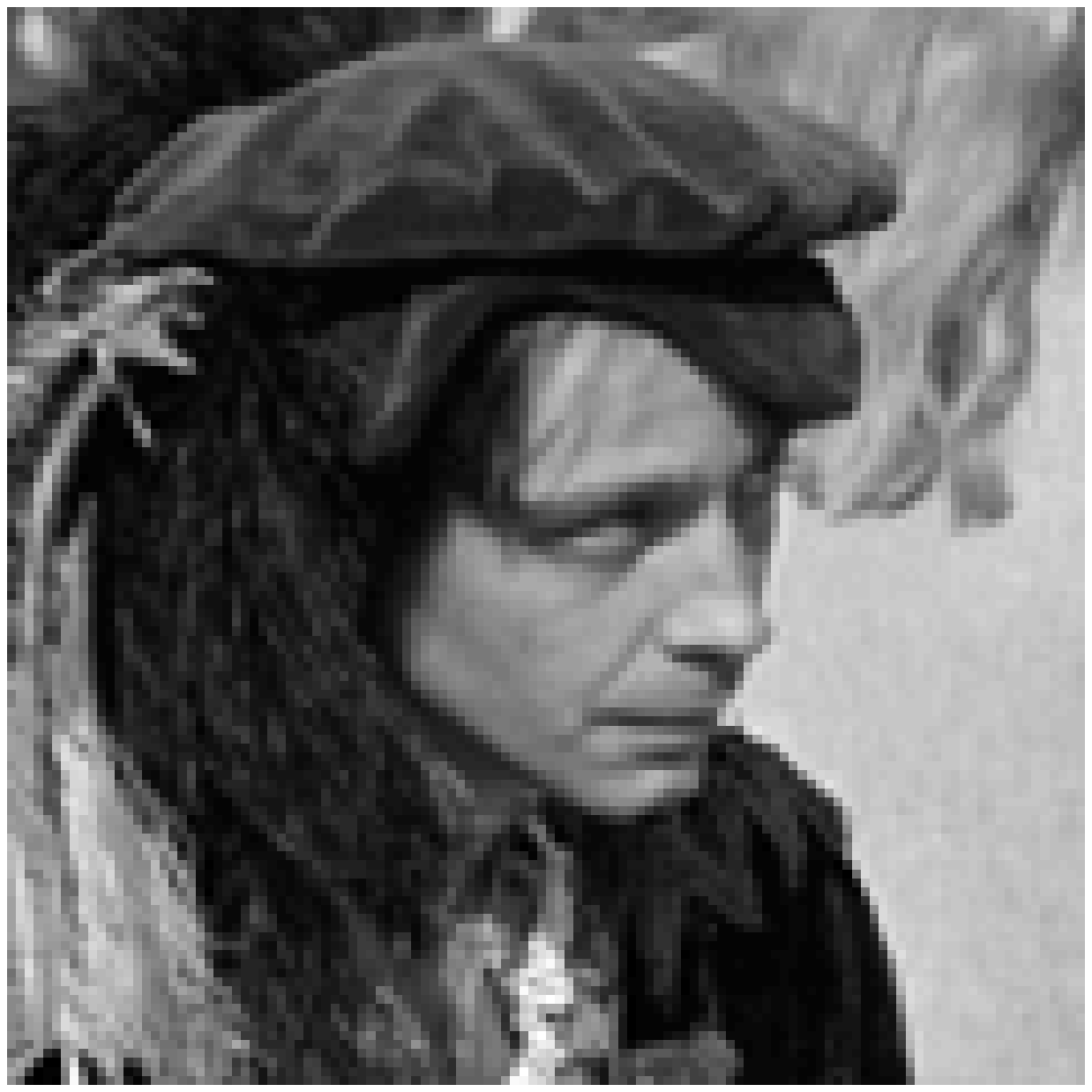}
}%
\subfloat[Couple\_2]{
\includegraphics[height=0.14\textheight]{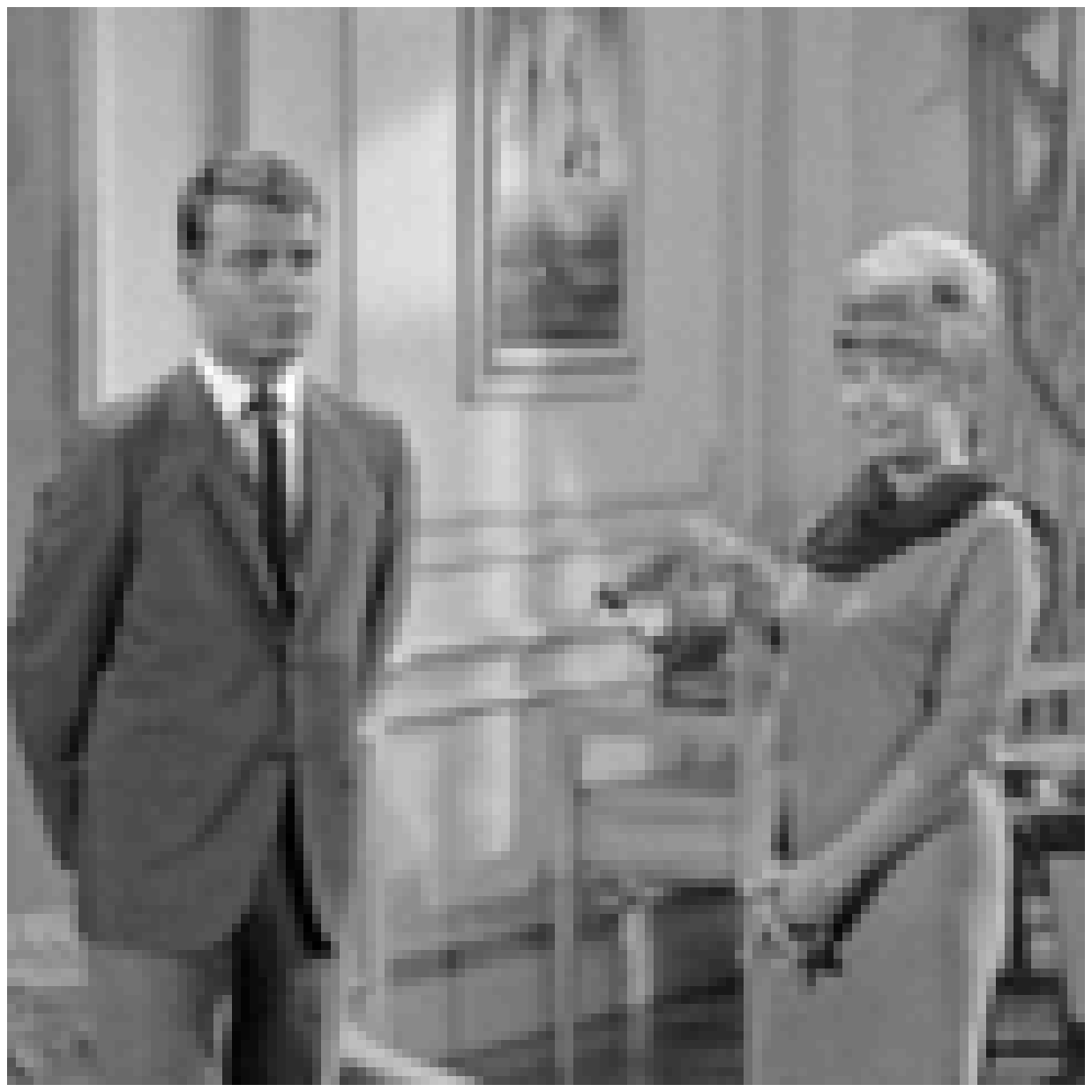}
}
\caption{Test images}
\label{fig:clean_images}
\end{figure}

The denoising performance in all of the following experiments is measured
by the peak signal-to-noise ratio (PSNR), which is given for images with values in the range
$0-255$ by
\begin{equation} \label{eqn:psnr}
 \PSNR\left(I_{1},I_{2}\right)=20\cdot \log_{10}\left(\frac{255}{\sqrt{\sum_{i,j}\left|I_{1}\left(i,j\right)-I_{2}\left(i,j\right)\right|^{2}}}\right) .
\end{equation}
In this metric, two similar images, for example an image and a good approximation of it, get high PSNR value, while two different images have low PSNR value.

In addition to the PSNR, which measures the difference between two images, we use the signal-to-noise ratio (SNR) to measure the noise level in a given (single) image. SNR is given by
\[ \SNR \left(I_c,\sigma_{noise}\right)=\frac{\sigma_{I_c}}{\sigma_{noise}} , \]
where $I_c$ is the clean image, $\sigma_{I_c}$ is the standard deviation of the intensity values of $I_c$ and $\sigma_{noise}$ is the standard deviation of the noise added to $I_c$.

The remaining of this section is organized as follows. First we examine the gain in using low-rank approximation, for different noise levels and kernel widths. Then, we compare between the naive approach of truncated spectral decomposition and our method with respect to the cutoff point, namely, the effect of different low-rank values on these two methods. We then proceed to comparing our method to a few popular denoising methods. Finally, we explore the performance of our method when combined in a state-of-the-art denoising scheme.

\subsection{Effectiveness of low-rank NLM operators} \label{sec:experiments_kernel_width}

We explore the improvement achieved by low-rank approximations of the NLM operator compared to the original NLM operator, as a function of the noise level (SNR) of the image and the kernel width of the operator.

For this experiment we used four of the test images given in Figure~\ref{fig:clean_images}, resized to $60 \times 60$ pixels by bi-cubic interpolation. We computed NLM operators for a range of kernel widths with patch size $p=5$. For each operator corresponding to a given kernel width, we constructed several low-rank approximations of it using the method of eigenvalues truncation, as given in Algorithm \ref{alg:nlm_eig} in Appendix \ref{sec:apx_algorithms}. The low-rank values are given in the set
\[
\begin{split}
K=\{1, 5, 10, 15, 20, 25, 50, 75, 100, 125, 150, 175, 200, \\
225, 250, 275, 300, 400, 600, 1200, 2000, 3000, 3600\}.
\end{split}
\]

In Figure~\ref{fig:kernel_width_and_order}, we present three charts for every image, corresponding to three SNR levels (for $4$ out of the $12$ images of Figure~\ref{fig:clean_images}). Each chart shows the PSNR between the clean image and its denoised version, as a function of the kernel width, for two operators. The two operators are the NLM operator and the best performing low-rank approximation, that is
\begin{equation} \label{eqn:optimal_low_rank}
  \arg \max_{k_i \in K}\left\lbrace \PSNR( \textrm{NLM-Eig}(I_c, p, k_i), I\right\rbrace,
\end{equation}
where $I$ and $I_c$ are the clean and corrupted images, respectively.

From the results of this experiment we can see that the performance gap between the optimal low-rank operator \eqref{eqn:optimal_low_rank} and the NLM operator can be very large, in favor of the low-rank approximation, for kernel widths that are much smaller than the ideal width. This advantage diminishes when the kernel width approaches its best performing value. Note that as the level of noise increases (the SNR value decreases) this optimal width value spreads to a broader range of values. In addition, one observes that the performance gain of the low-rank operator naturally diminishes as the SNR increases since less noise is involved.

\begin{figure}

\begin{centering}
Barbara
\par\end{centering}

\begin{centering}
\centerline{
\includegraphics[scale=1]{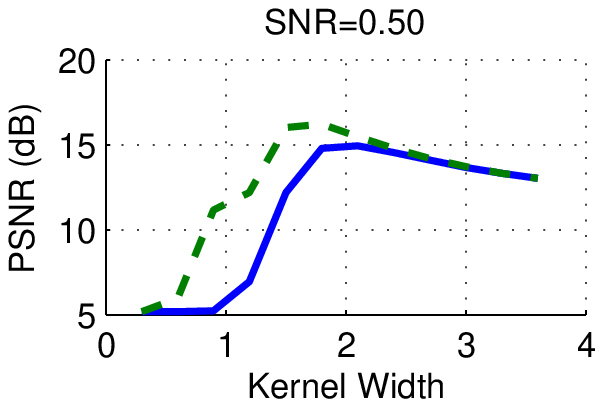}
\includegraphics[scale=1]{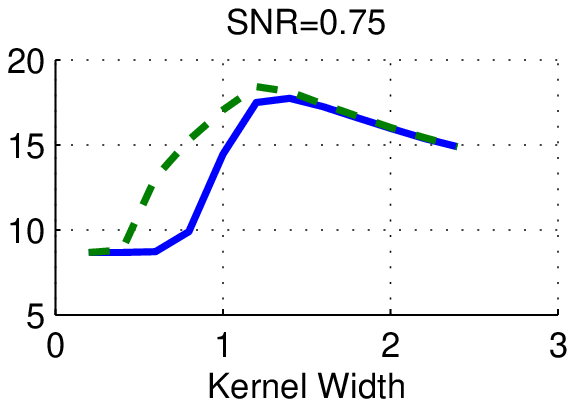}
\includegraphics[scale=1]{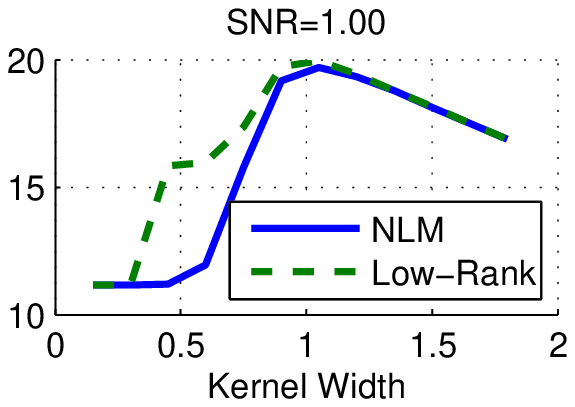}
}
\par\end{centering}

\begin{centering}
Clown
\par\end{centering}

\begin{centering}
\centerline{
\includegraphics[scale=1]{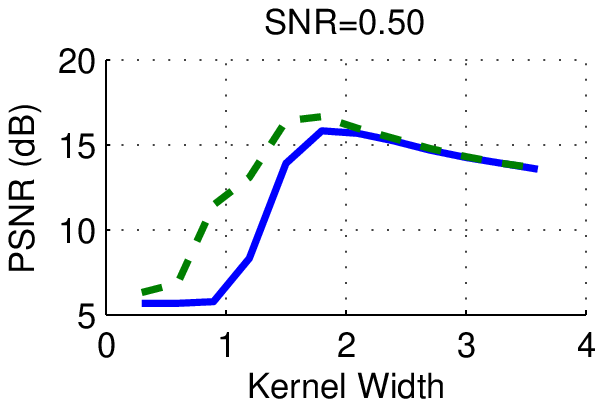}
\includegraphics[scale=1]{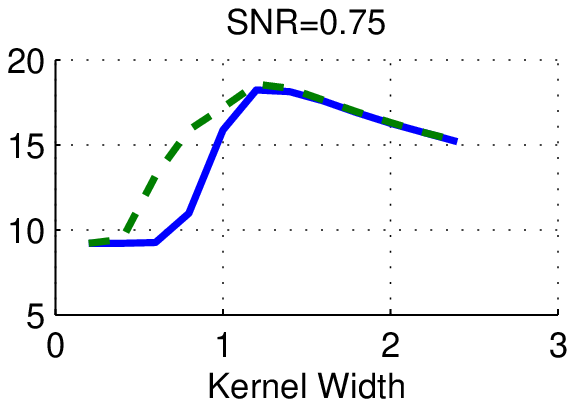}
\includegraphics[scale=1]{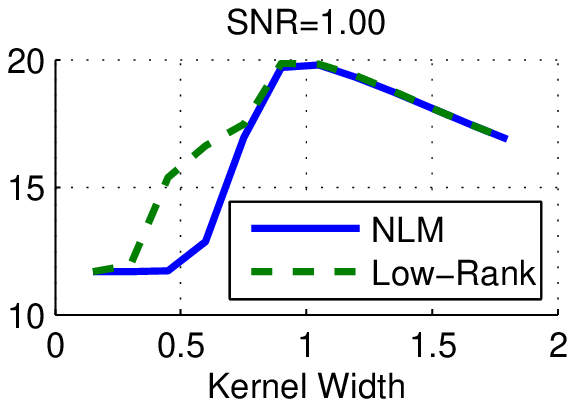}
}
\par\end{centering}

\begin{centering}
Mandril
\par\end{centering}

\begin{centering}
\centerline{
\includegraphics[scale=1]{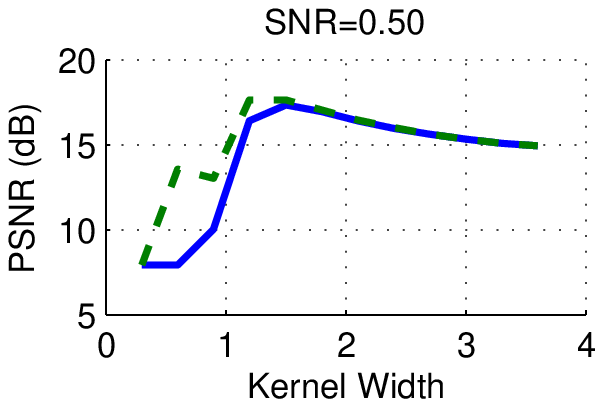}
\includegraphics[scale=1]{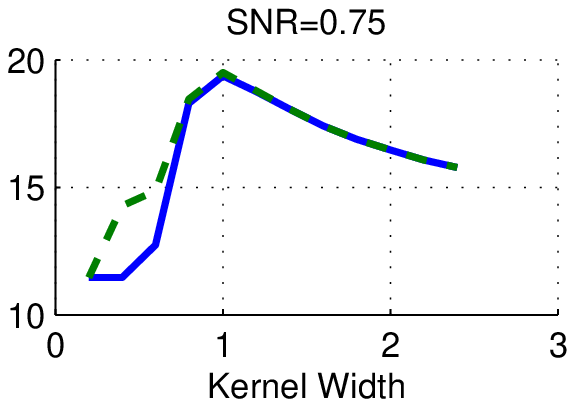}
\includegraphics[scale=1]{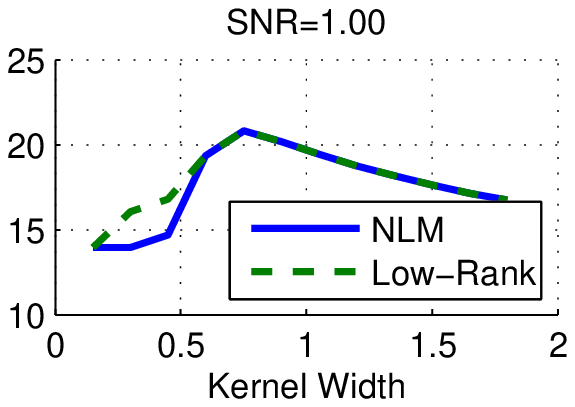}
}
\par\end{centering}

\begin{centering}
Roof
\par\end{centering}

\begin{centering}
\centerline{
\includegraphics[scale=1]{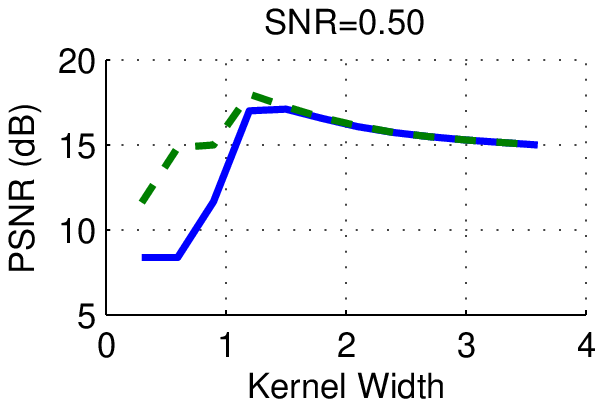}
\includegraphics[scale=1]{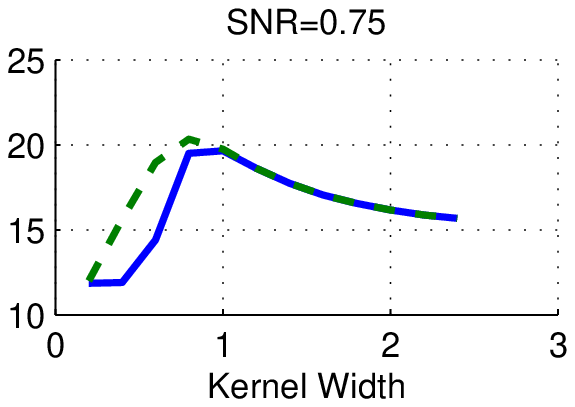}
\includegraphics[scale=1]{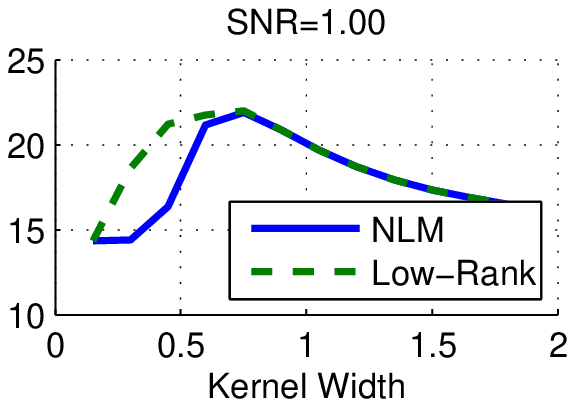}
}
\par\end{centering}

\caption{Using a low-rank operator versus the NLM operator: the PSNR between the clean image and its denoised version as a function of the kernel width, for different SNR (noise) levels. The original images are given in Figure \ref{fig:clean_images}.}
\label{fig:kernel_width_and_order}
\end{figure}

\subsection{The cutoff point} \label{sec:experiments_cutoff}

We investigate the effect of the cutoff point on the performance of two denoising methods. For the naive method of truncated eigenvalues (Algorithm \ref{alg:nlm_eig}), the cutoff point corresponds to the number of leading eigenvalues preserved in the low-rank operator. For our method based upon SB function (Algorithm \ref{alg:nlm_sbw}), the cutoff point corresponds to the filter's cutoff parameter $\omega$, which is a value between zero and one (Section \ref{subsec:construct_low_rank}). In the experiment of this section, we measure the ``PSNR gain", which is the difference between the PSNR of the low-rank operator and that of the original NLM operator it has been created from, as a function of the cutoff parameter.

The original NLM operators are constructed with a patch size $p=5$ and kernel widths $1.5$, $1.0$ and $0.5$ for the different  SNR values $0.5$, $0.75$ and $1$, respectively. These values were chosen based on the experiments in Section~\ref{sec:experiments_kernel_width}, as the kernel widths which result in the highest PSNR, on average. For our method based on the SB functions we use $N=150$ coefficients in the truncated Chebyshev expansion \eqref{eqn:cheb_mat_approx}. In addition, we select $d=15$ as the order of the filter \eqref{eqn:sbw} for the higher level of noise, that is SNR of $0.5$, while taking $d=4$ for the other two SNRs of $0.75$ and $1$.

Figures \ref{fig:psnr_trunc} and \ref{fig:psnr_sbw} show the results of the experiment, per image and noise level using eigenvalues truncation and our method, respectively. One can observe that for both methods, the best-performing cutoff, given the SNR, is far from being the same for all images. For cutoffs that are very low (resulting in extremely low-rank operators), the PSNR gain may be negative. For cutoffs that are very high, the PSNR of the low rank operator converges to the PSNR achieved by the original NLM operator, since the modified operator itself in both methods converges to the original NLM operator.

For a different perspective, we show in Figure~\ref{fig:mean_improvement} the mean improvement per noise level for the two methods. Indeed, there is a range of cutoffs that yields an improvement on the dataset for both methods.

\begin{figure}
\begin{center}
\subfloat{\includegraphics[width=.5\textwidth]{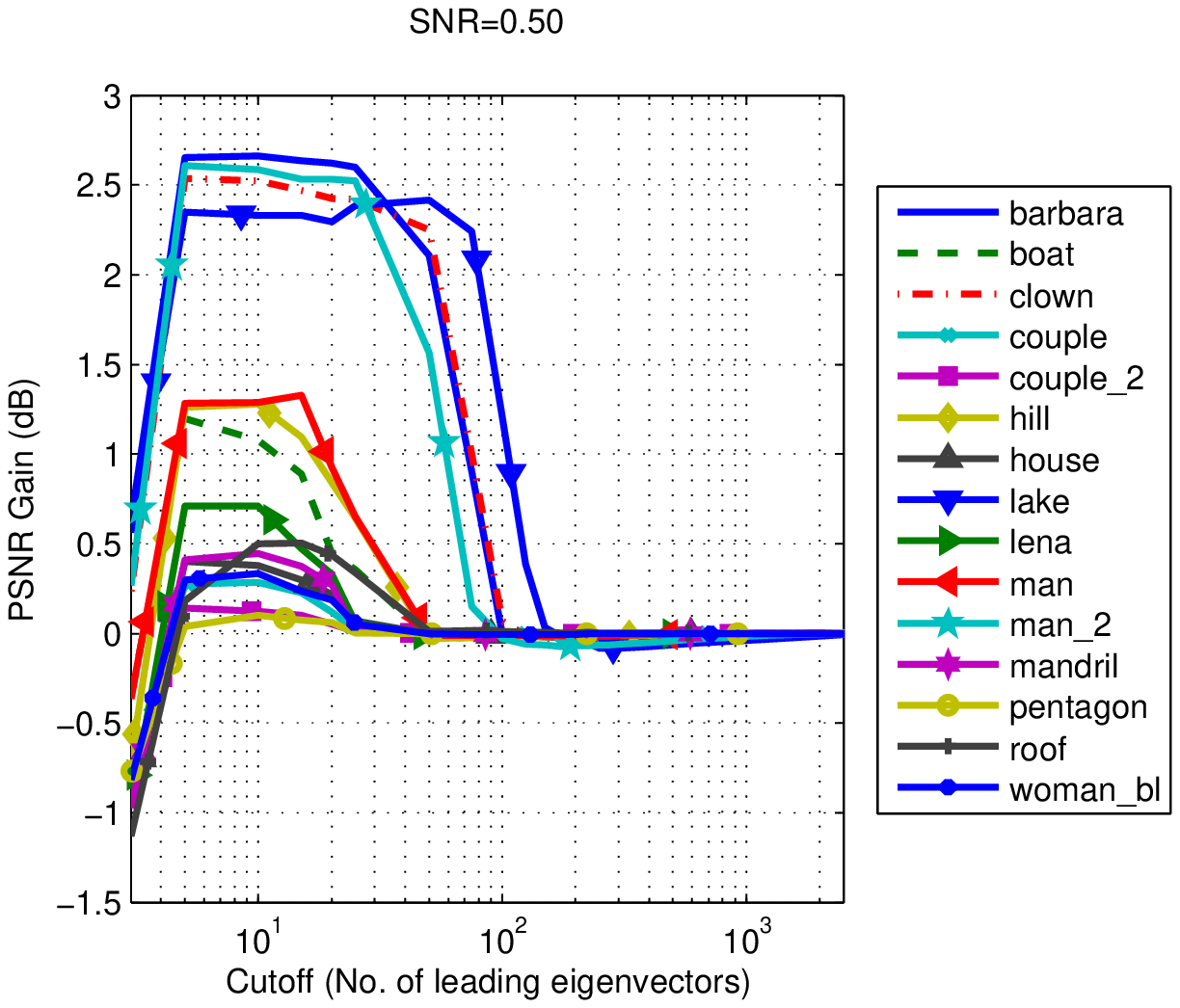}\label{subflaot:cutoff_snr_2}}
\subfloat{\includegraphics[width=.5\textwidth]{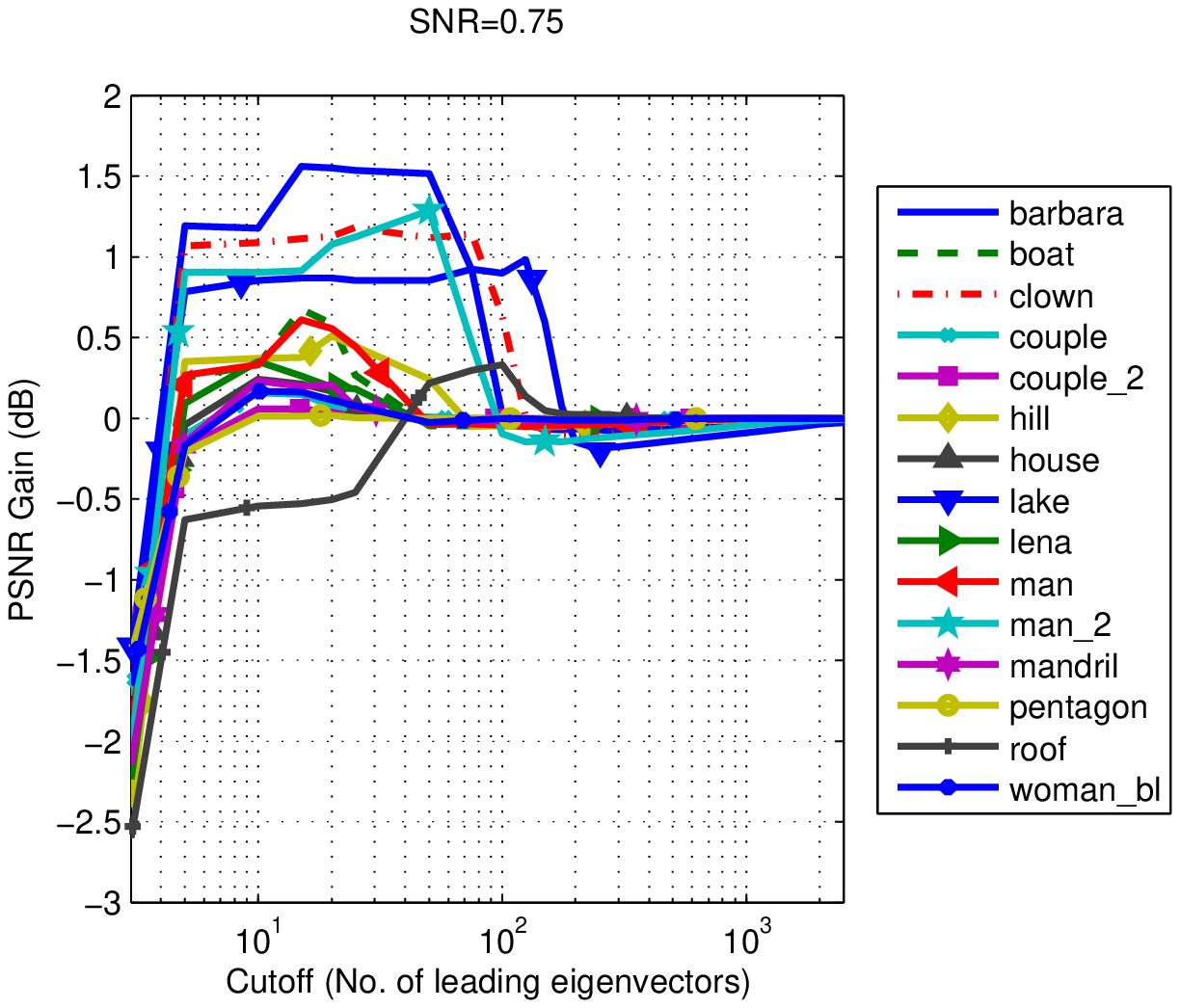}\label{subflaot:cutoff_snr_3}}  \\
\subfloat{\includegraphics[width=.5\textwidth]{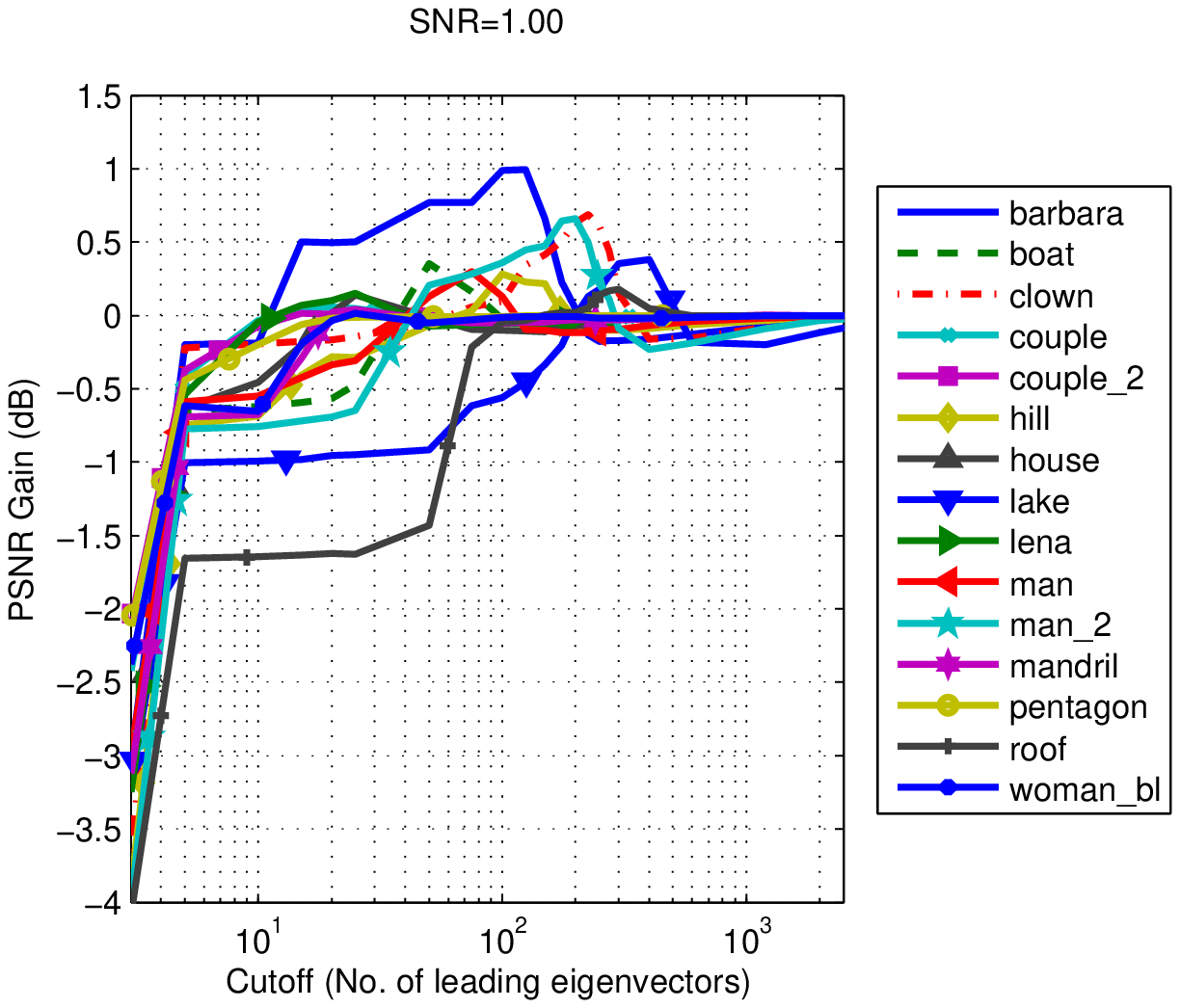}\label{subflaot:cutoff_snr_4}}   \\
\end{center}
\caption{The differences between the PSNR of the low-rank and the original NLM operator (PSNR gain), based on truncated eigenvalues, as a function of the cutoff point (number of preserved leading eigenvalues).}
\label{fig:psnr_trunc}
\end{figure}

\begin{figure}
\begin{center}
\subfloat{\includegraphics[width=.5\textwidth]{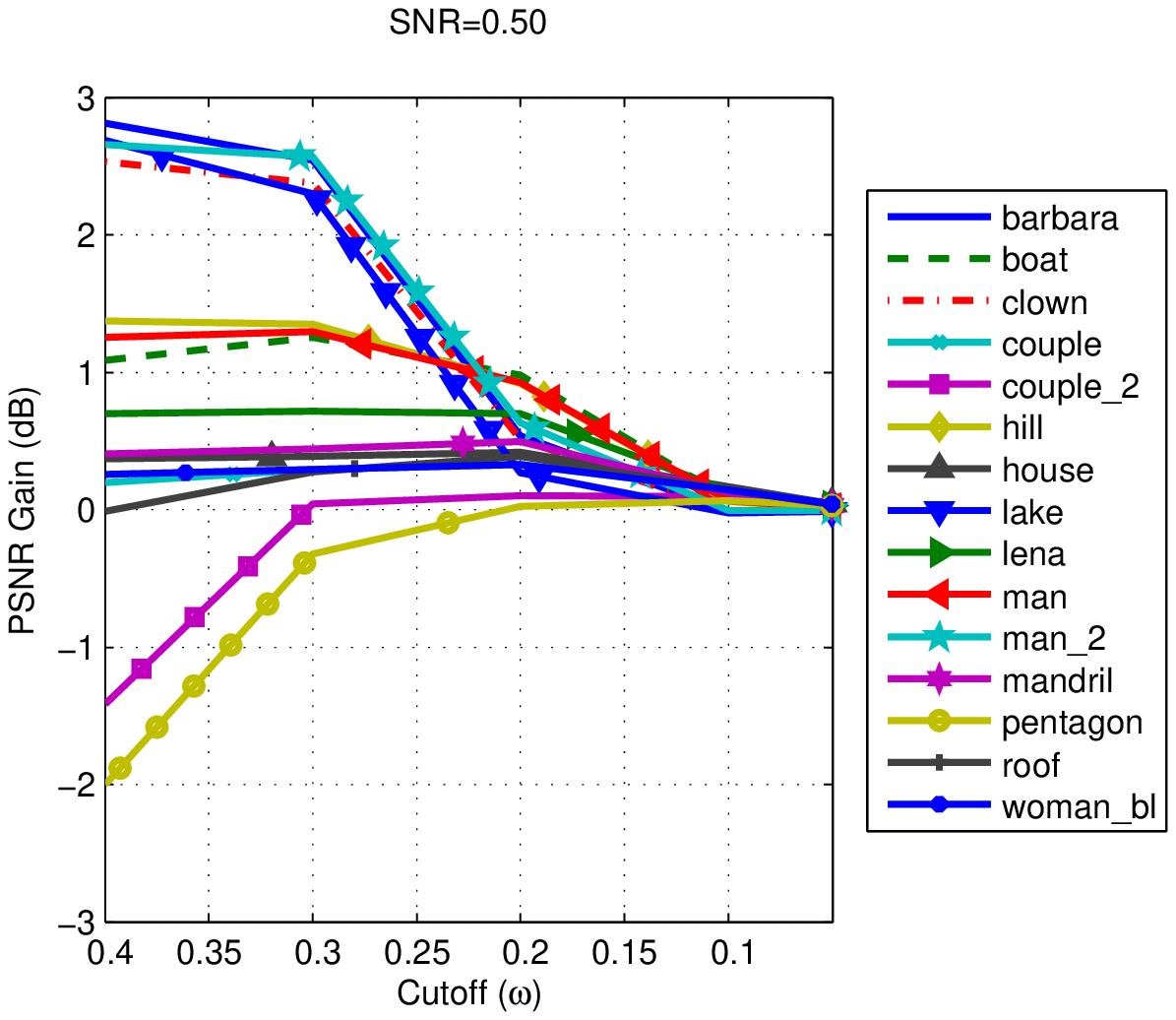}\label{subflaot:sbw_cutoff_snr_2}}
\subfloat{\includegraphics[width=.5\textwidth]{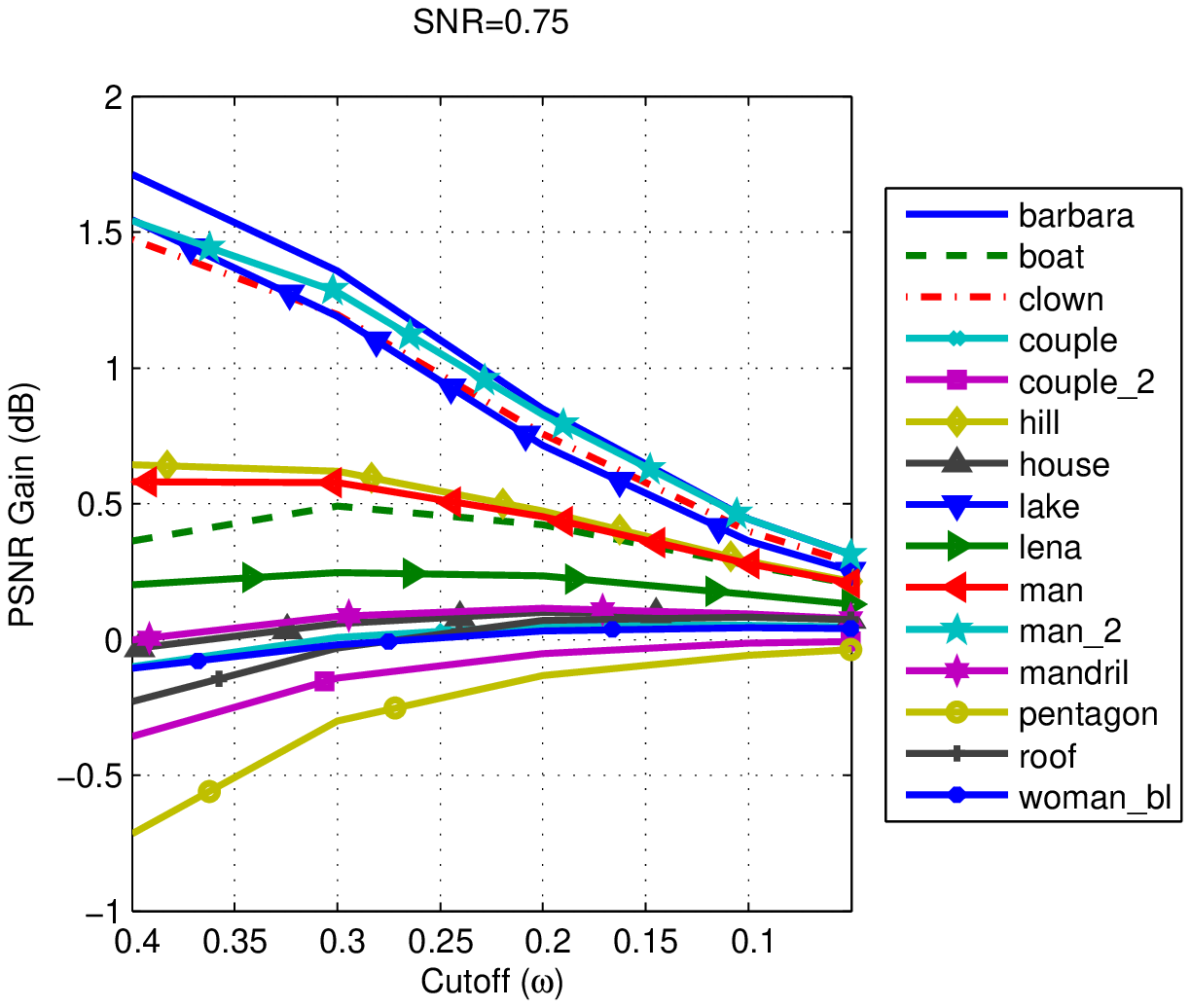}\label{subflaot:sbw_cutoff_snr_3}}  \\
\subfloat{\includegraphics[width=.5\textwidth]{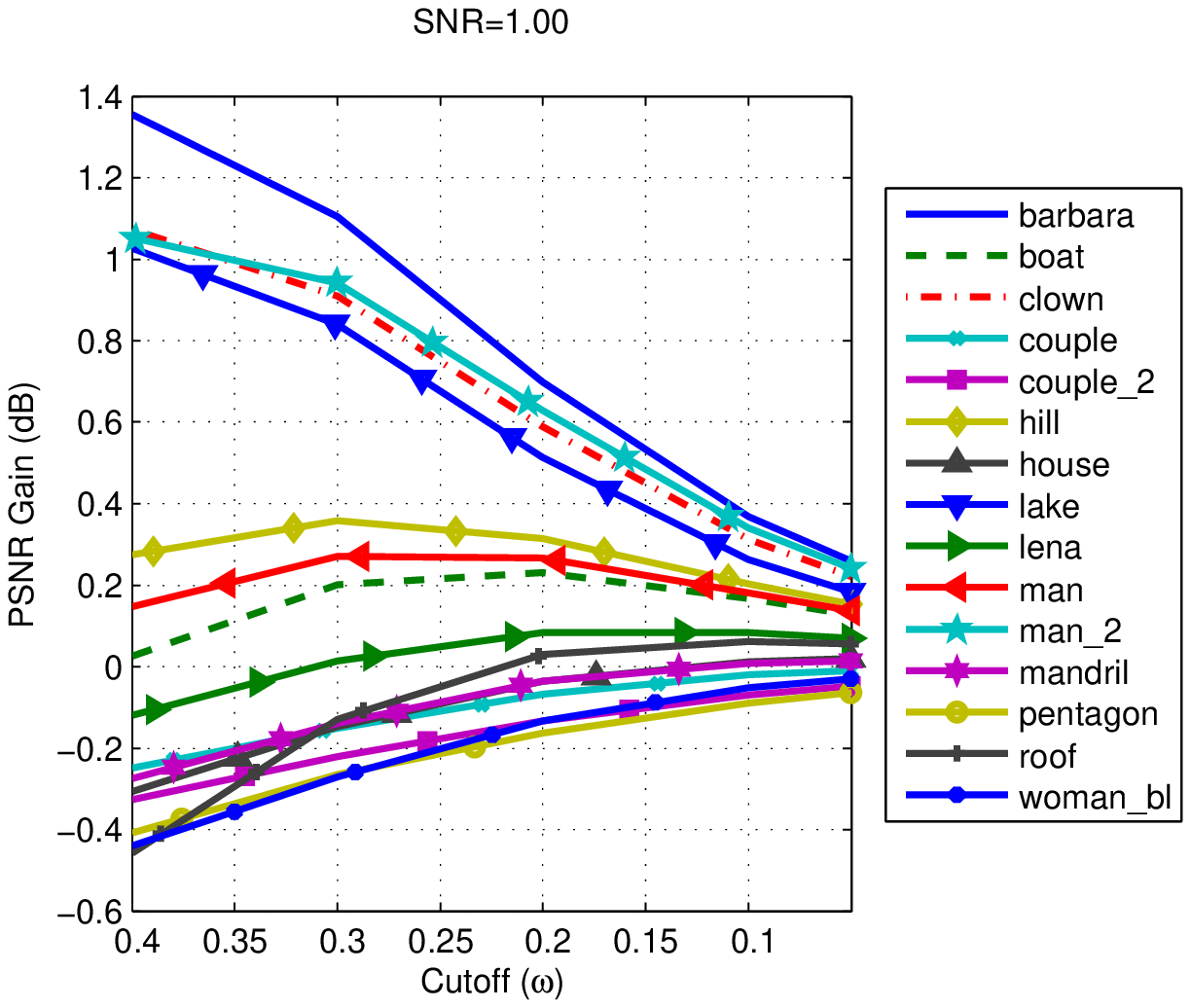}\label{subflaot:sbw_cutoff_snr_4}}   \\
\end{center}
\caption{The differences between the PSNR of the low-rank and the original NLM operator (PSNR gain), based on the SB function, as a function of the cutoff point (the parameter $\omega$).}
\label{fig:psnr_sbw}
\end{figure}

\begin{figure}
\begin{center}
\subfloat{\includegraphics[width=.5\textwidth]{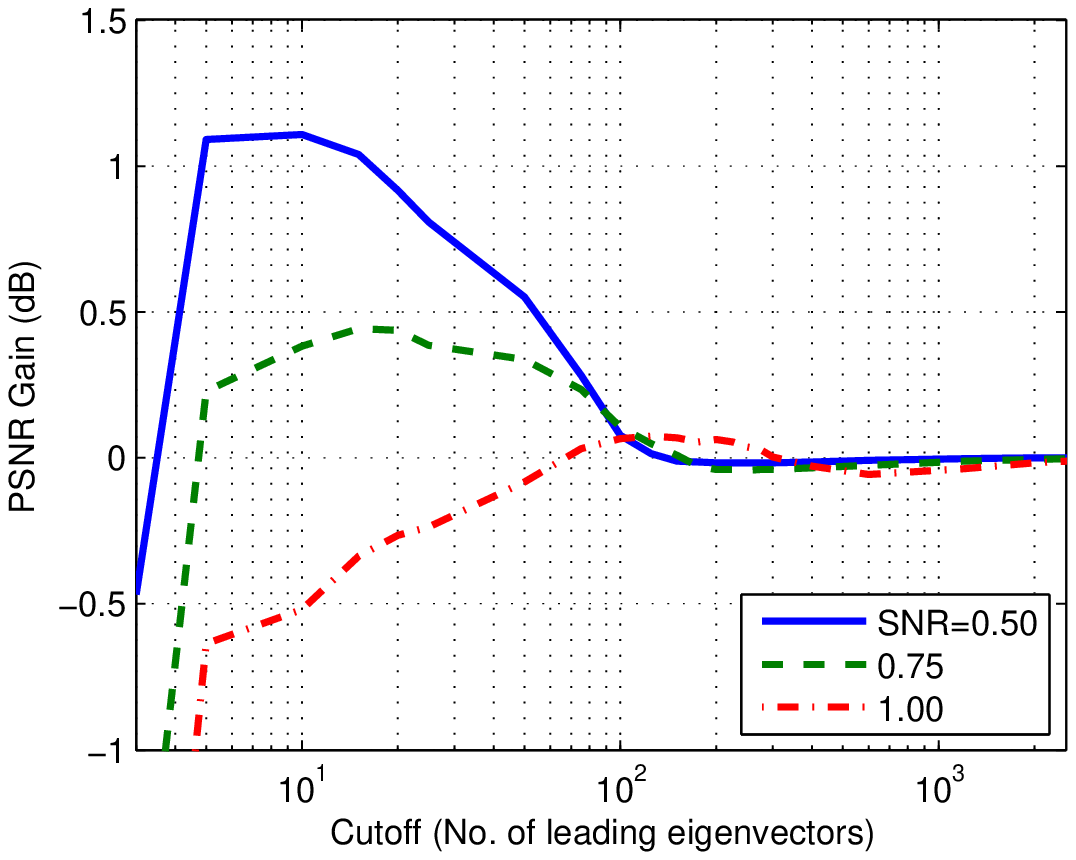}\label{subflaot:mean_cutoff}}
\subfloat{\includegraphics[width=.5\textwidth]{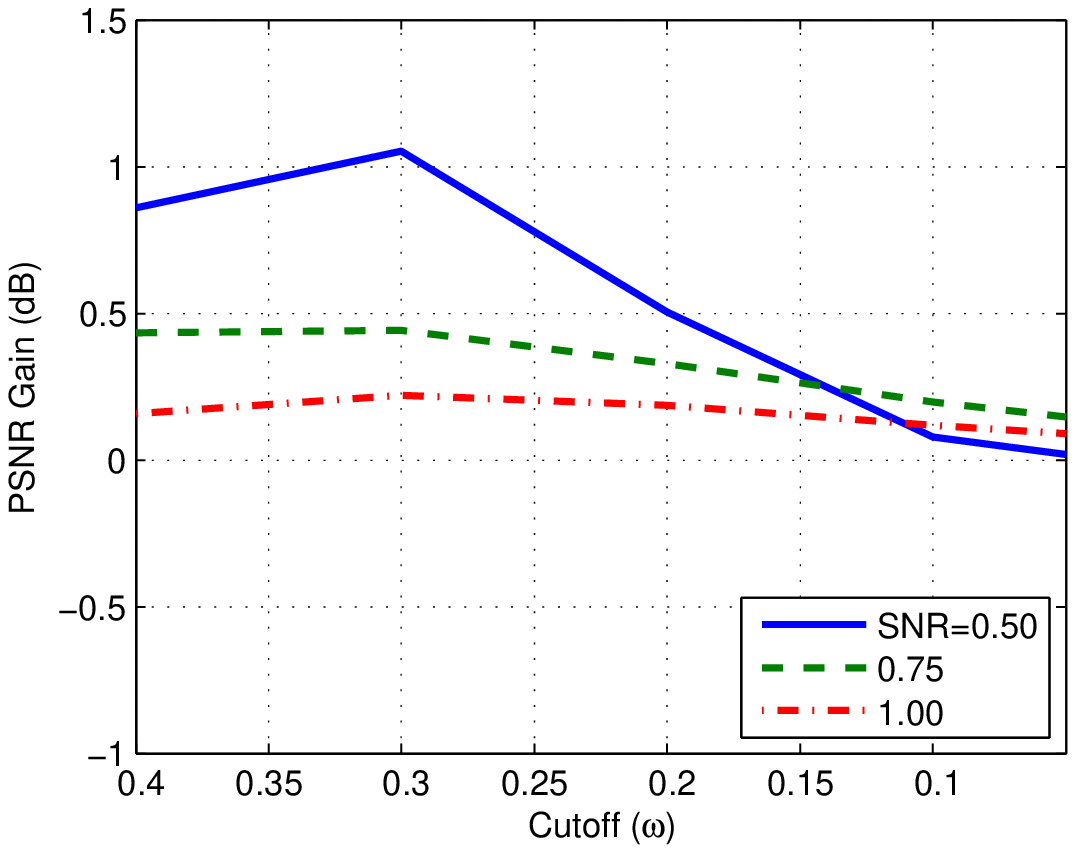}\label{subflaot:mean_sbw_cutoff}}
\end{center}
\caption{The average PSNR improvement over the full set of test images given in Figure~\ref{fig:clean_images} for three levels of noise. On the left, the results of the method of truncated eigenvalues (Algorithm \ref{alg:nlm_eig}). On the right, the results of our method, based on the SB function (Algorithm \ref{alg:nlm_sbw}). }
\label{fig:mean_improvement}
\end{figure}

\subsection{Comparison with other algorithms}

In Figure~\ref{fig:mean_improvement} we show the optimal average cutoff value for each noise level, averaged over the entire set of test images. In this section, we compare the results achieved using these cutoffs to the results of other denoising algorithms.

The denoising algorithms we compare are:  NLM, truncated eigenvalues NLM (NLM-Eig), our NLM based on the SB function (NLM-SB), NLM-SB2 (two stage denoising scheme to be described in Section~\ref{sec:twostagealg} below), K-SVD~\cite{elad2006image},  shape-adaptive DCT transform (SA-DCT)~\cite{foi2007pointwise}, and block-matching and 3D filtering (BM3D)~\cite{dabov2009bm3d}. The latter three have implementations which are publicly available. The parameters used for these algorithms are the default parameters given by the respective authors in their implementations.

Tables~\ref{tab:results_snr_05} and \ref{tab:results_snr_075} contain PSNR comparisons of the above algorithms for the images of Figure \ref{fig:clean_images}, including a summary line consisting of the average PSNR over the full set of images.

\begin{table}
\centerline{%
\begin{tabular}{cccccccc}
Image & NLM & NLM-Eig & NLM-SB & NLM-SB2 & K-SVD & SA-DCT & BM3D\tabularnewline\hline
barbara & 15.47 & 18.13 & 18.14 & 19.37 & 20.10 & 21.09 & 20.88\tabularnewline
boat & 16.98 & 18.06 & 18.20 & 19.13 & 18.26 & 20.40 & 20.81\tabularnewline
clown & 15.48 & 18.00 & 17.94 & 18.70 & 18.14 & 19.66 & 19.85\tabularnewline
couple & 17.97 & 18.25 & 18.25 & 18.69 & 20.03 & 21.00 & 21.24\tabularnewline
couple\_2 & 18.62 & 18.74 & 18.70 & 18.91 & 19.79 & 22.20 & 22.58\tabularnewline
hill & 16.56 & 17.84 & 17.86 & 18.86 & 18.54 & 20.09 & 20.30\tabularnewline
house & 18.15 & 18.52 & 18.52 & 19.27 & 18.95 & 22.38 & 23.44\tabularnewline
lake & 14.81 & 17.14 & 17.24 & 18.36 & 17.75 & 18.98 & 18.83\tabularnewline
lena & 17.52 & 18.23 & 18.21 & 19.01 & 18.90 & 20.61 & 20.99\tabularnewline
man & 16.57 & 17.86 & 17.82 & 18.83 & 17.88 & 19.56 & 19.75\tabularnewline
man\_2 & 15.82 & 18.41 & 18.46 & 19.51 & 19.23 & 20.37 & 20.49\tabularnewline
mandril & 17.30 & 17.75 & 17.72 & 18.37 & 17.84 & 20.05 & 20.27\tabularnewline
pentagon & 18.07 & 18.17 & 18.04 & 17.96 & 15.56 & 20.55 & 21.27\tabularnewline
roof & 17.33 & 17.83 & 17.52 & 18.15 & 17.19 & 20.64 & 21.80\tabularnewline
woman\_blonde & 18.02 & 18.36 & 18.31 & 18.85 & 18.00 & 20.77 & 21.05\tabularnewline\hline\hline
\textbf{Average} & 16.98 & 18.09 & 18.06 & 18.80 & 18.41 & 20.56 & 20.90
\end{tabular}}
\caption{PSNR values of the tested algorithms for noise level corresponding to SNR=0.5.}
\label{tab:results_snr_05}
\end{table}

\begin{table}
\centerline{%
\begin{tabular}{cccccccc}
Image & NLM & NLM-Eig & NLM-SB & NLM-SB2 & K-SVD & SA-DCT & BM3D\tabularnewline\hline
barbara & 18.83 & 20.39 & 20.65 & 21.27 & 20.87 & 22.49 & 22.90\tabularnewline
boat & 19.68 & 20.35 & 20.15 & 21.13 & 18.64 & 21.92 & 22.22\tabularnewline
clown & 18.66 & 19.77 & 20.16 & 20.37 & 19.00 & 21.32 & 21.62\tabularnewline
couple & 20.06 & 20.21 & 19.99 & 21.86 & 19.93 & 22.52 & 22.85\tabularnewline
couple\_2 & 20.48 & 20.54 & 20.33 & 22.38 & 20.88 & 23.78 & 24.15\tabularnewline
hill & 19.41 & 19.79 & 20.00 & 21.25 & 18.97 & 21.72 & 21.94\tabularnewline
house & 20.78 & 20.99 & 20.76 & 22.74 & 20.75 & 24.67 & 25.58\tabularnewline
lake & 18.05 & 18.92 & 19.55 & 18.70 & 18.22 & 20.48 & 20.49\tabularnewline
lena & 20.02 & 20.28 & 20.14 & 21.35 & 19.43 & 22.26 & 22.57\tabularnewline
man & 19.45 & 20.06 & 19.97 & 20.66 & 18.28 & 21.08 & 21.30\tabularnewline
man\_2 & 19.13 & 20.04 & 20.74 & 21.06 & 19.44 & 21.98 & 22.16\tabularnewline
mandril & 19.50 & 19.70 & 19.44 & 20.76 & 19.51 & 21.15 & 21.39\tabularnewline
pentagon & 19.51 & 19.52 & 19.31 & 20.82 & 19.05 & 22.31 & 22.90\tabularnewline
roof & 19.84 & 19.30 & 19.53 & 20.71 & 19.98 & 22.12 & 23.73\tabularnewline
woman\_blonde & 20.31 & 20.47 & 20.15 & 21.26 & 19.80 & 22.55 & 22.87\tabularnewline\hline\hline
\textbf{Average} & 19.58 & 20.02 & 20.06 & 21.09 & 19.52 & 22.16 & 22.58
\end{tabular}}
\caption{PSNR values of the tested algorithms for noise level corresponding to SNR=0.75.}
\label{tab:results_snr_075}
\end{table}

\subsection{Two stages denoising schemes}\label{sec:twostagealg}

Meyer and Shen~\cite{meyer2014perturbation} have proposed to compute a second NL-means operator from the denoised image and apply it again.
Except for introducing two stages, there are a few additional differences between their algorithm and the NL-means algorithm, given in Definition~\ref{def:nlmeans_op}.

First, it contains a $k$-nearest neighbours search for constructing the denoising operators. Instead of constructing the operator $A=D^{-1}W$, where the elements of $W$ are given as in Definition~\ref{def:nlmeans_op}, they construct $W$ as $W=0.5(Z+Z')$, where the matrix $Z$ contains $k$ nonzero elements in row $i$, which are the $k$ largest values of the set
\[ \left\lbrace K_h(v(N_i^Y)-v(N_j^Y)) \mid 1 \leq j \leq n \right\rbrace . \]
Second, the way that the denoising operator is applied is different than in the NLM algorithm. Instead of computing the denoised image as $\hat{x}=Ay$,  where each pixel is denoised using only the patch of which it is the center, they average the values of the pixel from all patches in which it is included.

For comparison purposes, we have modified Meyer's scheme to use our low-rank operator, based on the SB function. The resulting algorithm is henceforth referred to as SB-Meyer. In order to use $f^{sb}_{\omega,d}$ in this scheme, we had to modify the way the denoising operator is computed; In contrary to the NLM operator, the operator in Meyer's scheme is not semi-positive definite (not all eigenvalues are non-negative). This is due to using $k$-nearest neighbours to construct it. Since our framework, as given in Section~\ref{sec:spectrum shaping}, requires a non-negative spectrum, we have approximated the denoising operator in Meyer's scheme with a new semi-positive definite matrix, by shifting its spectrum such that the largest negative eigenvalue becomes zero. Then, we normalized its rows such that they will sum to unity. This normalization transforms the largest positive eigenvalue back to unity (see also the proof of Lemma \ref{lemma:NLM_matrix_properties} in Appendix \ref{sec:apn_proof_of_lemma}).

The original algorithm of \cite{meyer2014perturbation} uses parameters given by the authors in their source code, which were tuned only for SNR level of $1$. Thus, we report in Table~\ref{tab:results_snr_1} its results (in the column ``Meyer'') and also those of SB-Meyer only for that noise level. The comparison is over the full set of images (as given in Figure \ref{fig:clean_images}), where the parameters of SB-Meyer are given in Table~\ref{tab:sbw_meyer_params}. In addition to SB-Meyer, we have implemented a simpler two-stage scheme employing our operator based on the SB function, which is given in Algorithm~\ref{alg:sbw_twostage} and is referred to as NLM-SB2. Its parameters are given in Table~\ref{tab:nlm_sbw2_params}. From Table~\ref{tab:results_snr_1} we see that using the method of the current paper improves the resulting PSNR (on average).  Using a two stage scheme further improves the resulting PSNR.

\begin{landscape}
\begin{table}
\centerline{%
\begin{tabular}{cccccccccc}
Image & NLM & NLM-Eig & NLM-SB & NLM-SB2 & SB-Meyer & Meyer & K-SVD & SA-DCT & BM3D\tabularnewline\hline
barbara & 20.84 & 21.83 & 21.94 & 22.11 & 23.18 & 20.30 & 21.67 & 23.61 & 24.07\tabularnewline
boat & 21.24 & 21.16 & 21.44 & 22.01 & 22.72 & 21.33 & 21.40 & 23.08 & 23.22\tabularnewline
clown & 20.51 & 20.85 & 21.42 & 21.08 & 22.23 & 20.14 & 20.40 & 22.54 & 22.85\tabularnewline
couple & 21.66 & 21.65 & 21.51 & 22.54 & 23.50 & 23.06 & 22.27 & 23.86 & 24.17\tabularnewline
couple\_2 & 22.03 & 22.03 & 21.81 & 23.17 & 24.46 & 24.61 & 22.20 & 25.17 & 25.37\tabularnewline
hill & 21.27 & 21.50 & 21.63 & 22.20 & 22.57 & 21.21 & 21.38 & 23.00 & 23.15\tabularnewline
house & 22.83 & 22.72 & 22.68 & 23.89 & 24.82 & 26.19 & 24.10 & 26.31 & 27.18\tabularnewline
lake & 19.73 & 19.28 & 20.57 & 19.24 & 21.42 & 20.78 & 19.77 & 21.66 & 21.66\tabularnewline
lena & 21.73 & 21.68 & 21.75 & 22.39 & 23.12 & 22.44 & 21.63 & 23.51 & 23.74\tabularnewline
man & 21.18 & 21.09 & 21.45 & 21.70 & 22.04 & 21.16 & 19.94 & 22.32 & 22.41\tabularnewline
man\_2 & 20.94 & 21.38 & 21.88 & 21.47 & 22.08 & 20.01 & 20.02 & 23.14 & 23.29\tabularnewline
mandril & 21.08 & 21.03 & 20.94 & 21.45 & 22.37 & 21.50 & 20.61 & 22.11 & 22.24\tabularnewline
pentagon & 20.84 & 20.84 & 20.57 & 21.61 & 23.50 & 24.22 & 21.62 & 23.67 & 24.12\tabularnewline
roof & 21.97 & 21.93 & 21.84 & 22.66 & 22.57 & 24.42 & 21.64 & 23.46 & 25.04\tabularnewline
woman\_blonde & 21.87 & 21.86 & 21.60 & 21.89 & 23.18 & 23.64 & 23.17 & 24.03 & 24.22\tabularnewline\hline\hline
\textbf{Average} & 21.31 & 21.39 & 21.53 & 21.96 & 22.92 & 22.33 & 21.45 & 23.43 & 23.78\tabularnewline
\end{tabular}}
\caption{Algorithm comparison for SNR=$1$}
\label{tab:results_snr_1}
\end{table}
\end{landscape}

\begin{table}
\begin{center}
\begin{tabular}{  c  | cccccccccccc }
& $v_1$ & $v_2$ & $p_1$ & $p_2$ & $h_1$ & $h_2$ & $\omega_1$ & $\omega_2$ & $d_1$ & $d_2$ &  $\gamma$ & $k$  \\ \hline
 SNR=1 & 200 & 200 & 7 & 3 & 1 & 1 & 0.6 & 0.4 & 50 & 16 & 0.2 & 150
\end{tabular}
\caption{The parameters of the SB-Meyer scheme.}
\label{tab:sbw_meyer_params}
\end{center}
\end{table}

\begin{table}
\begin{center}
\begin{tabular}{  c  | ccccccccc }
SNR & $p$ & $h_1$ & $h_2$ & $\omega_1$ & $\omega_2$ & $d_1$ & $d_2$ & $\gamma$ & $N$  \\ \hline
 0.5 & 5& 1.5 & 1& 0.3&0.3& 50& 50& 0.5& 150 \\
 0.75 & 5& 1.05 & 0.35& 0.3&0.3& 15& 15& 0.15& 150 \\
 1 & 5& 0.5 & 0.3& 0.3&0.3& 4& 4& 0.15& 150
 \end{tabular}
\caption{The parameters of the NLM-SB2 scheme, as given in Algorithm~\ref{alg:sbw_twostage}.}
\label{tab:nlm_sbw2_params}
\end{center}
\end{table}

For a more visual perspective of the comparison, we provide Figures~\ref{fig:mandril_denoised} and~\ref{fig:barbara_denoised}, where two images are denoised by the various tested algorithms. The first image is part of the Mandril image, whose clean version is in Figure \ref{fig:mandril_image}. Its denoised versions are presented in Figures~\ref{fig:mandril_nlm}--\ref{fig:mandril_bm3d}. One can see that SB-Meyer retained the texture much better than BM3D, and it also contains less artifacts than Meyer's original scheme (the finding regarding texture preservation is consistent with the conclusions in~\cite{meyer2014perturbation}). The second image is part of the Barbara image, shown in Figure \ref{fig:barbara_image}. The denoising results for the Barbara image are presented in Figures~\ref{fig:barbara_nlm}--\ref{fig:barbara_bm3d}, where we can see that the outputs of NLM-SB and NLM-SB2 are considerably less noisy than the output of NLM while retaining nearly the same amount of details. In addition, it appears that some small image regions that were incorrectly estimated by NLM were also incorrectly estimated by NLM-SB. A few of these artifacts are visible on the Barbara image and they are marked by red rectangles in Figures \ref{fig:barbara_image_arti1} and \ref{fig:barbara_image_arti2}. This effect of a remaining noisy region was amplified by the application of a second denoising stage in NLM-SB2. We believe that these artifacts can be removed by some simple heuristic, but this is outside the scope of this work.

\begin{figure}
\begin{center}
\subfloat[Clean]{
\includegraphics[width=0.22\textwidth]{mandril_clean} \label{fig:mandril_image}
}%
\subfloat[Noisy]{
\includegraphics[width=0.22\textwidth]{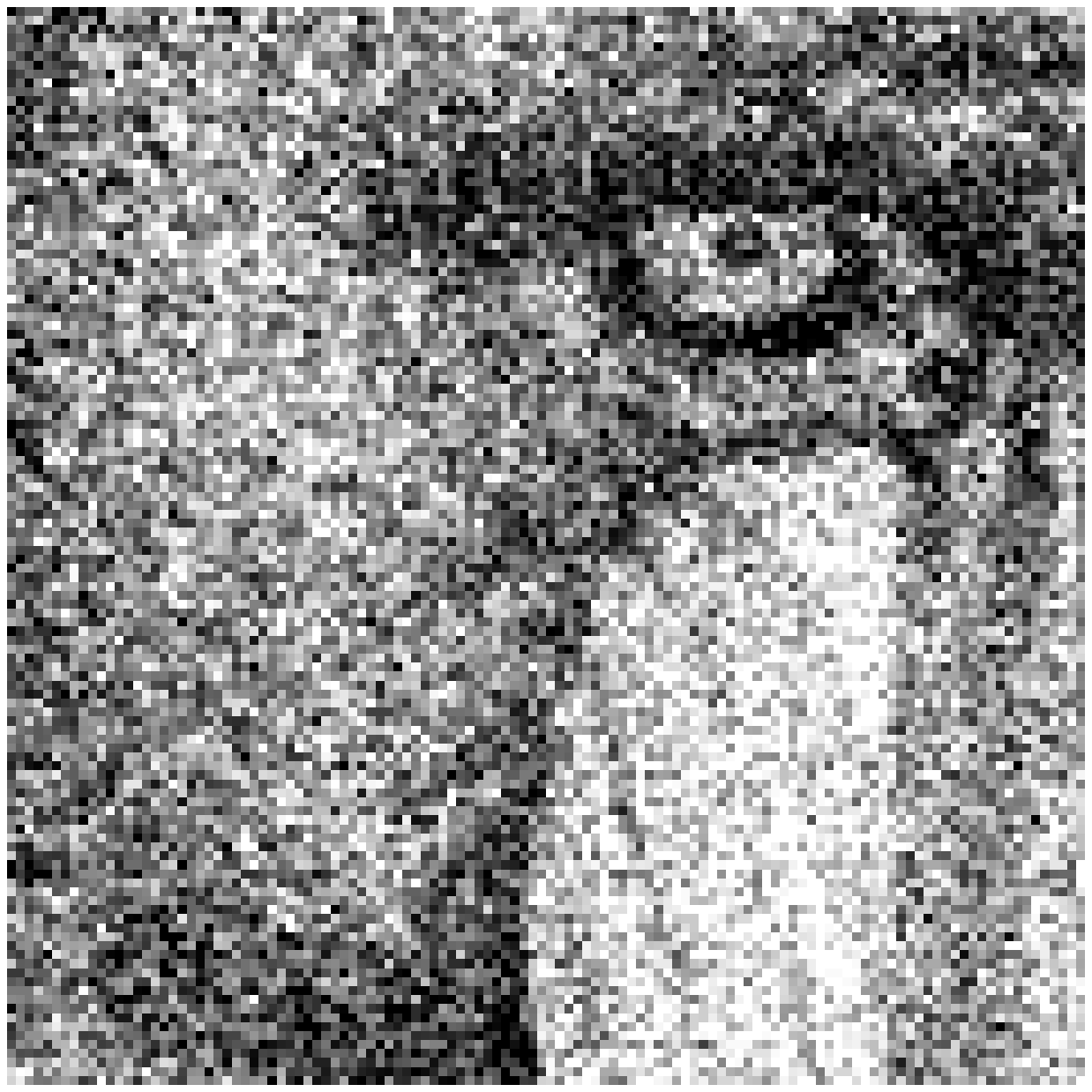}
}%
\subfloat[NLM]{
\includegraphics[width=0.22\textwidth]{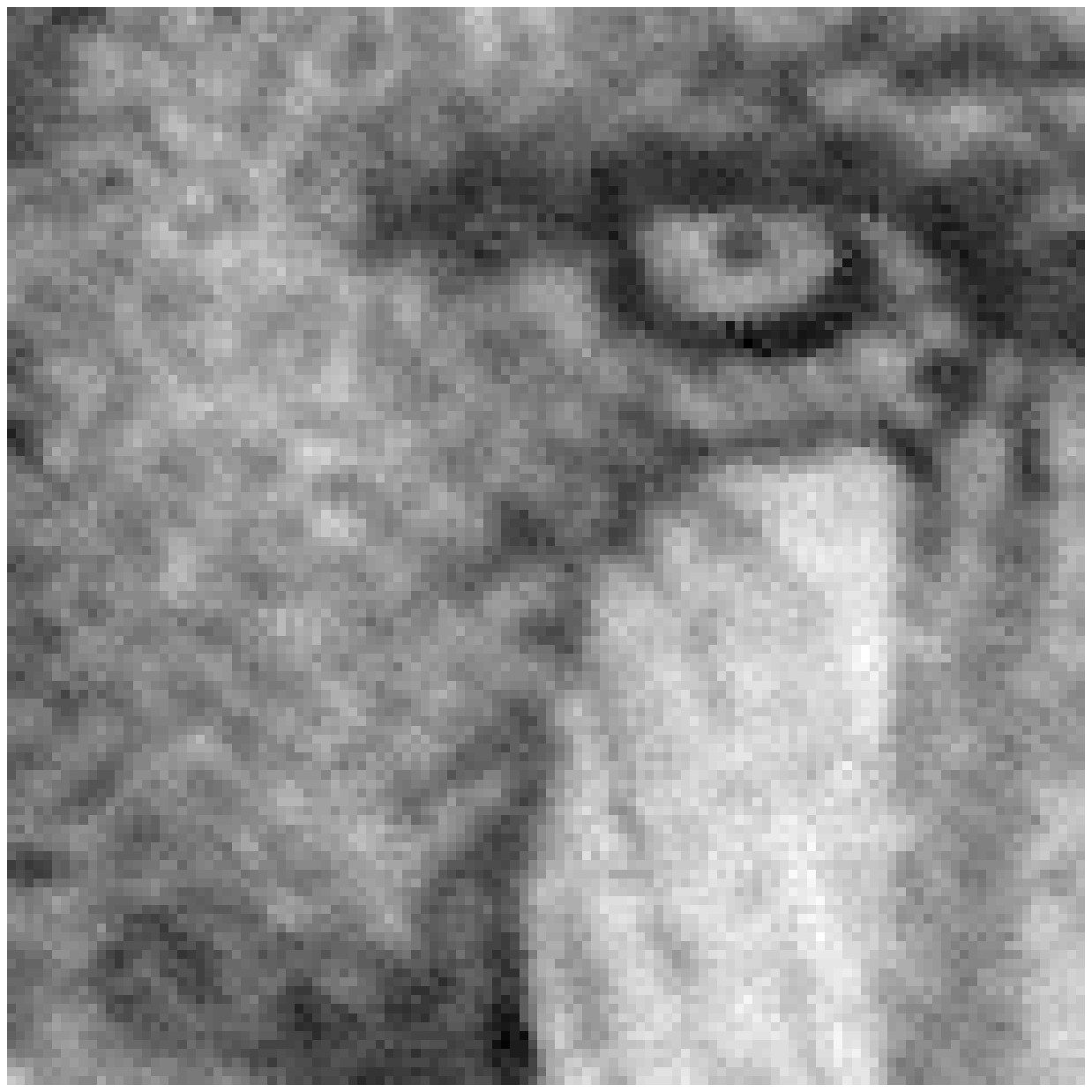}
\label{fig:mandril_nlm}
}\\
\subfloat[SB-Meyer]{
\includegraphics[width=0.22\textwidth]{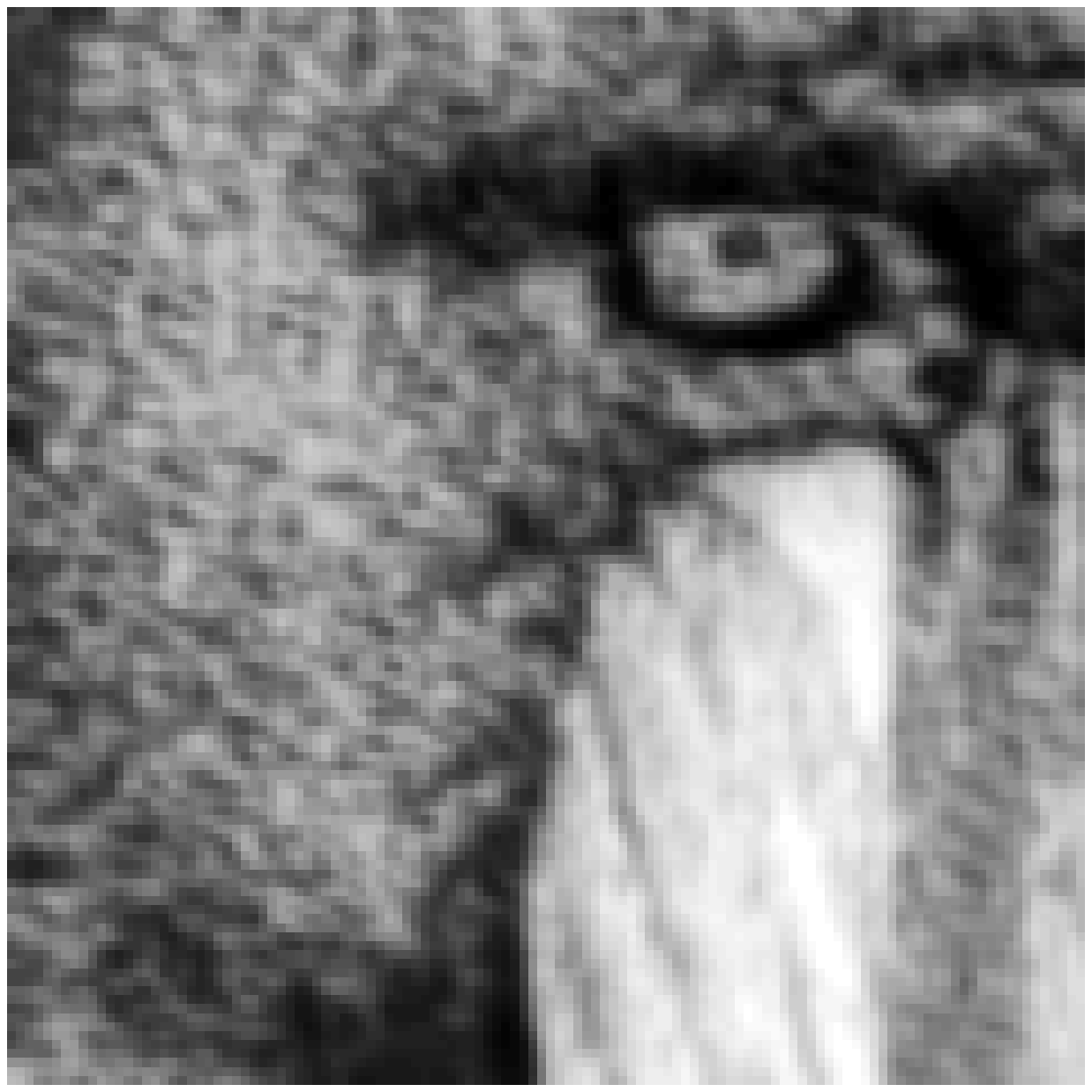}
}%
\subfloat[Meyer]{
\includegraphics[width=0.22\textwidth]{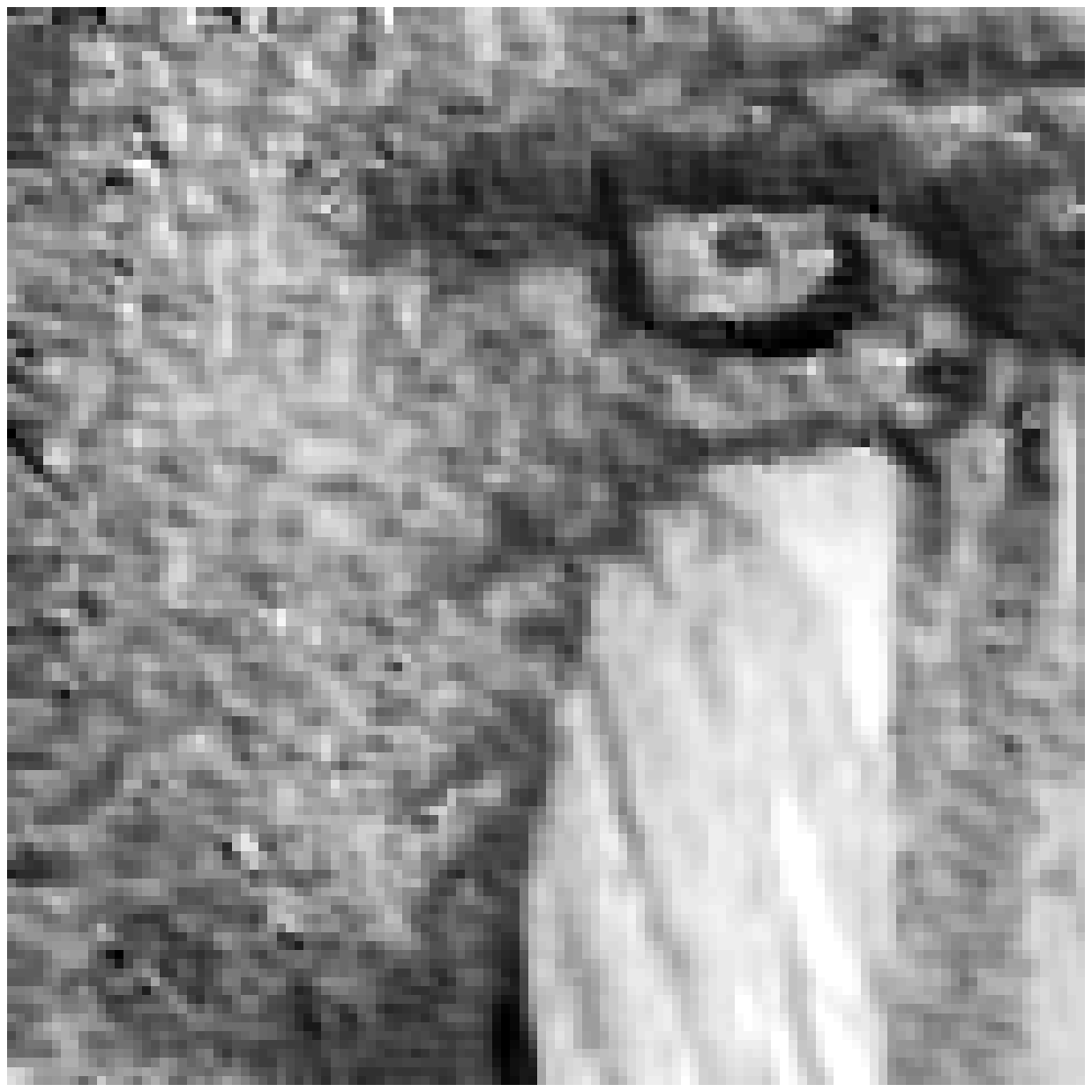}
}%
\subfloat[BM3D]{
\includegraphics[width=0.22\textwidth]{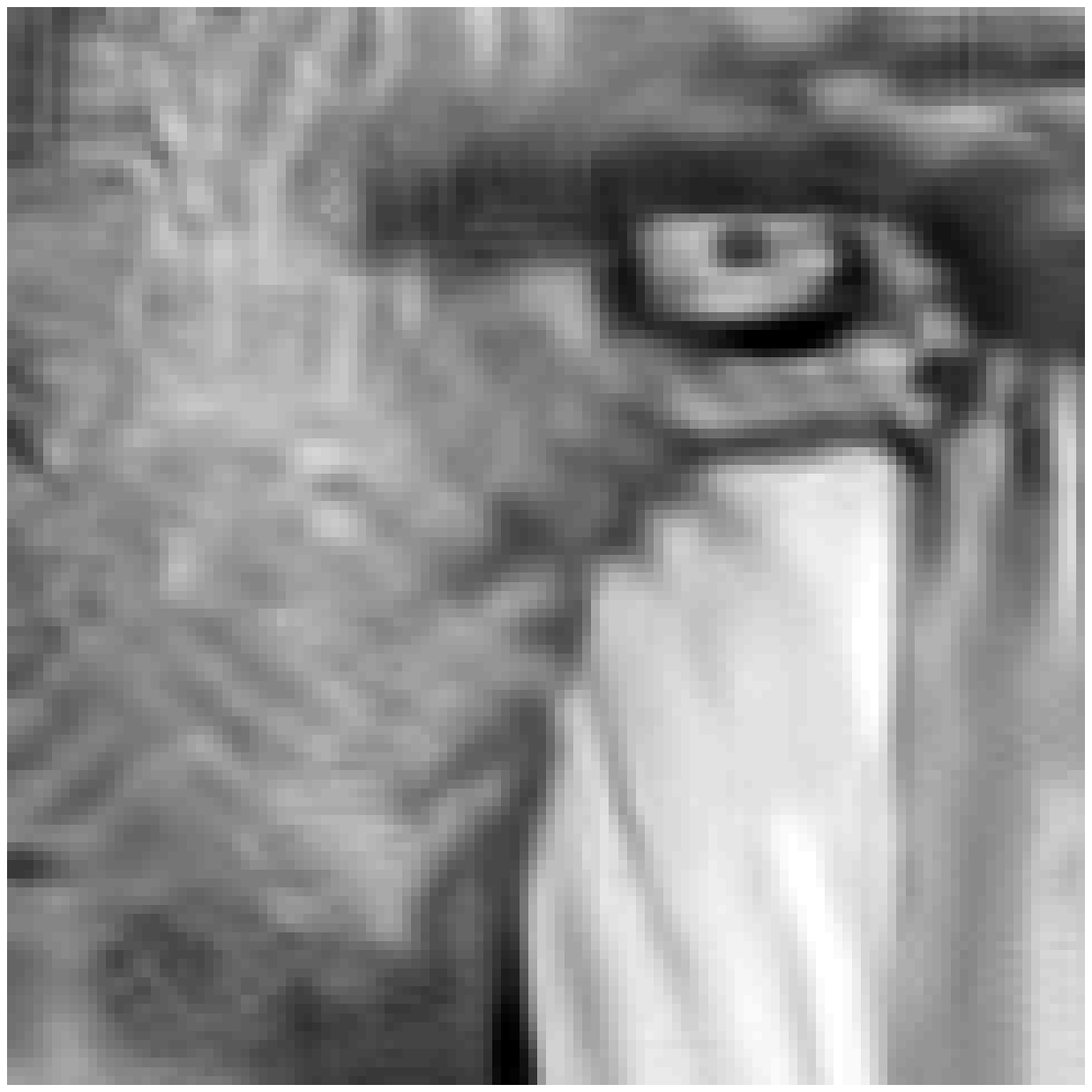}
\label{fig:mandril_bm3d}
}
\caption{Denoised examples taken from the Mandril image with SNR$=1$.}
\label{fig:mandril_denoised}
\end{center}
\end{figure}

\begin{figure}
\begin{center}
\subfloat[Clean]{
\includegraphics[width=0.22\textwidth]{barbara_clean} \label{fig:barbara_image}
}%
\subfloat[Noisy]{
\includegraphics[width=0.22\textwidth]{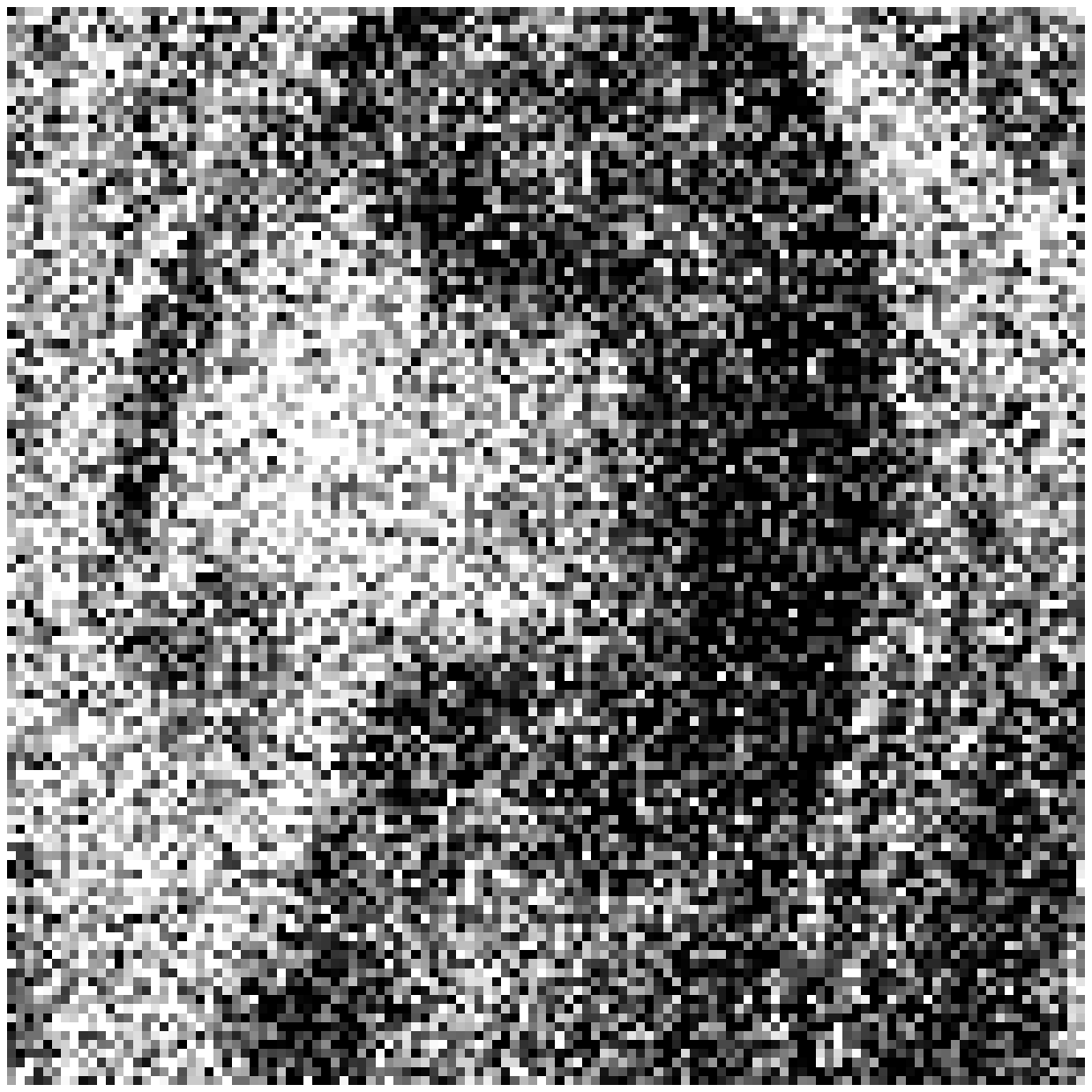}
}%
\subfloat[NLM]{
\includegraphics[width=0.22\textwidth]{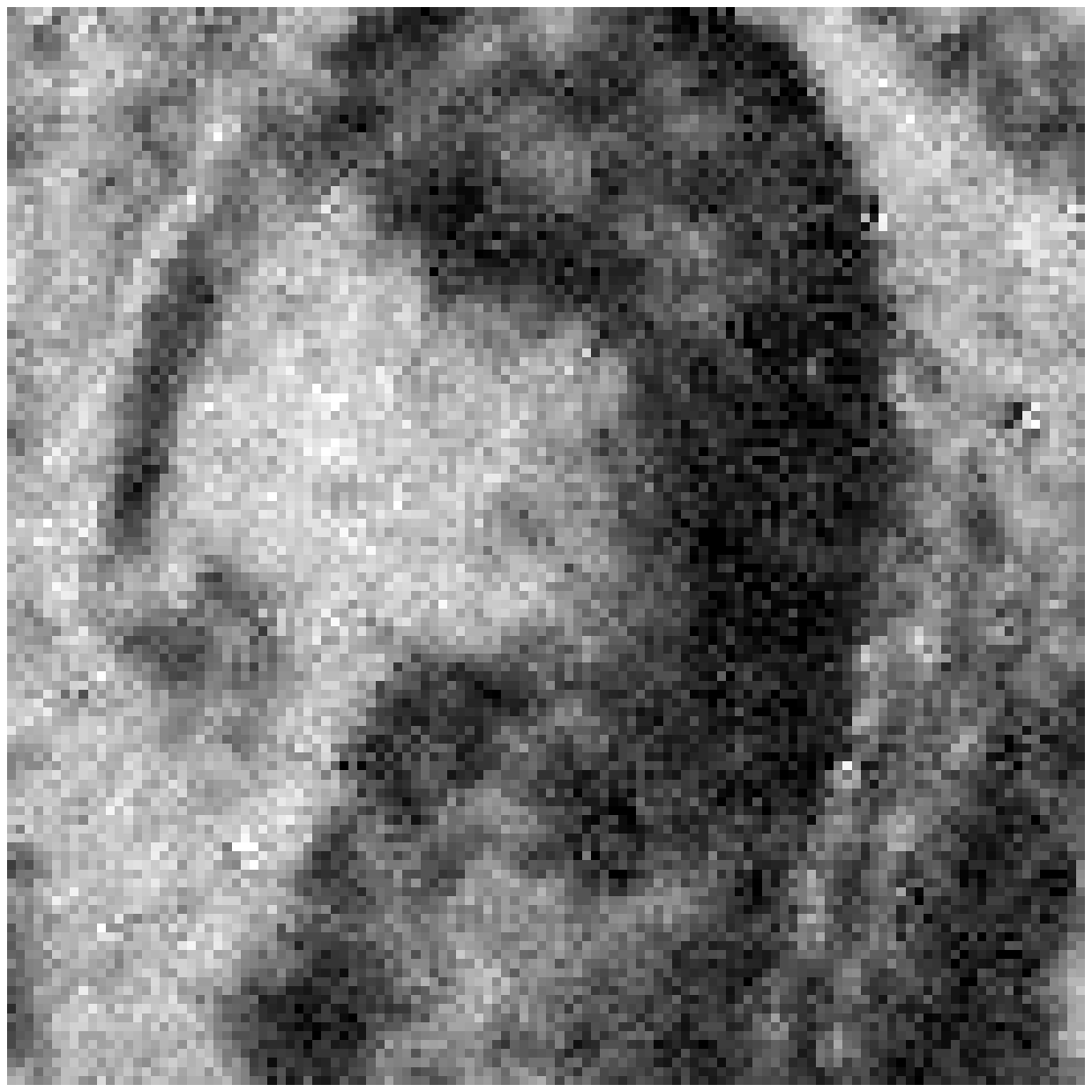}
\label{fig:barbara_nlm}
}\\
\subfloat[NLM-SB]{
\includegraphics[width=0.22\textwidth]{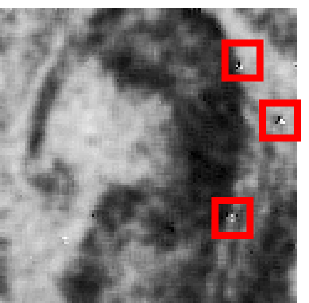} \label{fig:barbara_image_arti1}
}%
\subfloat[NLM-SB2]{
\includegraphics[width=0.22\textwidth]{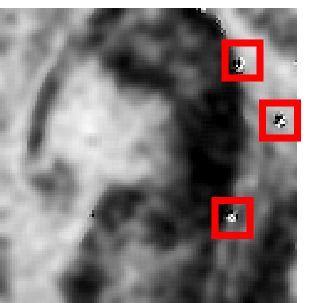} \label{fig:barbara_image_arti2}
}%
\subfloat[BM3D]{
\includegraphics[width=0.22\textwidth]{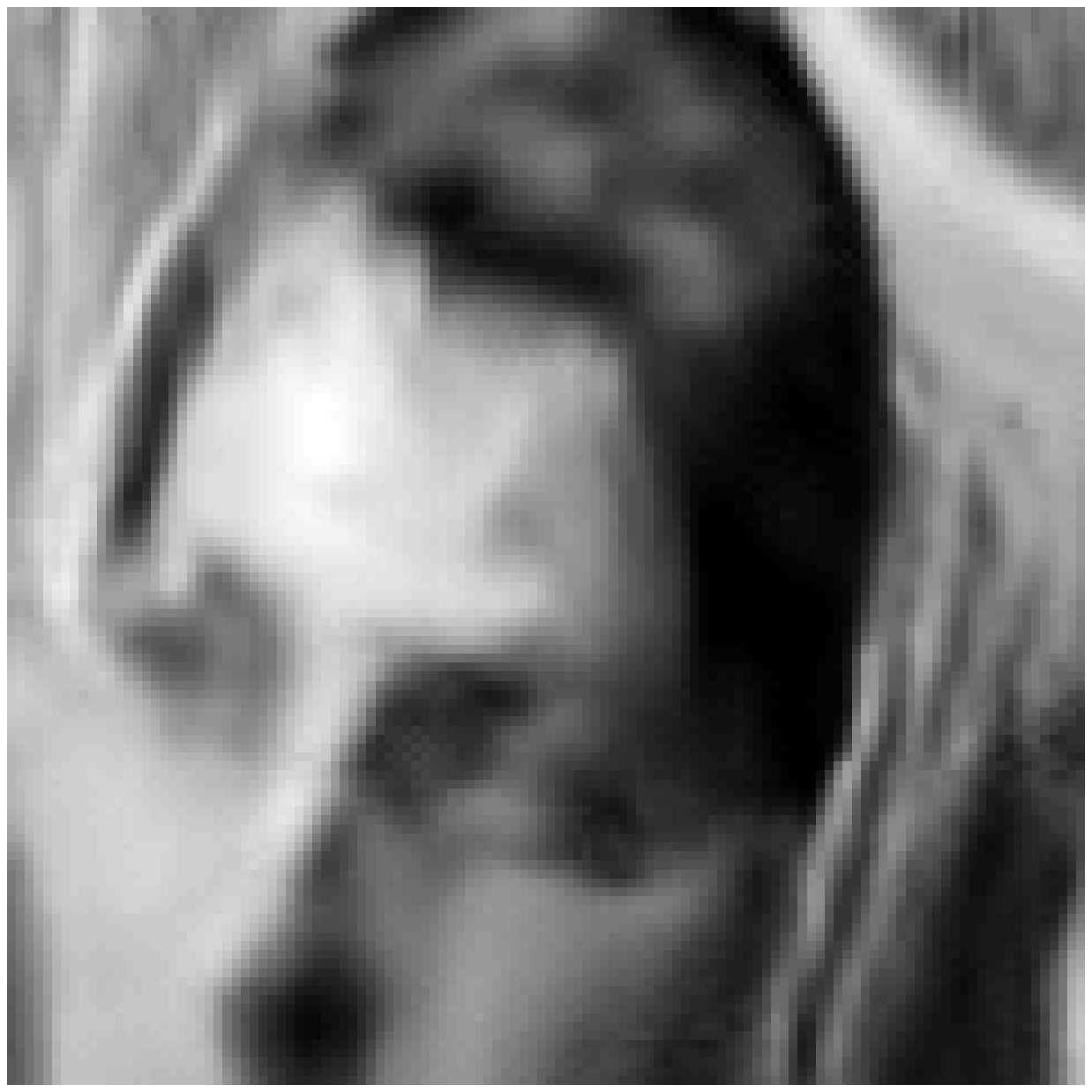}
\label{fig:barbara_bm3d}
}
\caption{Denoised examples taken from the Barbara image with SNR$=0.75$. Remaining noise artifacts are marked with red rectangles.}
\label{fig:barbara_denoised}
\end{center}
\end{figure}

\section{Summary}\label{sec:conclusions}

In this paper, we have investigated the idea of improving a Non-Local Means operator by manipulating its spectrum. We have shown a method
to do so without computing explicitly neither the eigenvectors of the original operator nor the matrix of the modified operator.

Our method operates by applying a filtering function to the original Non-Local Means operator. To that end, we have derived sufficient conditions
for such filtering functions and an efficient procedure for its application.

In this work we also show the connection between spectral shaping of the operator and the application of a matrix function on the operator. In the implementation of our approach, we demonstrate the well-known efficiency of the Chebyshev polynomials for matrix functions. Moreover, a bound on the approximation error of the truncated Chebyshev expansion for the class of Non-Local Means matrices is proved.

We present numerical experiments from which we learn a few important observations about the improvement that can be gained by a low-rank approximation of a Non-Local Means operator. First, the improvement depends on the choice of the kernel width parameter and the noise level. Second, we find that the optimal cutoff of the spectrum of the Non-Local Means operator varies for different images and noise levels. Nevertheless, the methods of eigenvalue truncation and SB filtering, as given in Algorithms  \ref{alg:nlm_eig} and \ref{alg:nlm_sbw} respectively, both achieve a non-negligible improvement, on average, on our dataset of $15$ images. It is also demonstrated that a further PSNR improvement can be achieved by running SB filtering via a two-stage scheme, as suggested in \cite{meyer2014perturbation}.

\bibliographystyle{plain}
\bibliography{mybib}

\begin{thebibliography}{10}

\bibitem{buades2005non}
Antoni Buades, Bartomeu Coll, and J-M Morel.
\newblock A non-local algorithm for image denoising.
\newblock In {\em Computer Vision and Pattern Recognition, 2005. CVPR 2005.
  IEEE Computer Society Conference on}, volume~2, pages 60--65. IEEE, 2005.

\bibitem{butterworth1930theory}
Stephen Butterworth.
\newblock On the theory of filter amplifiers.
\newblock {\em Wireless Engineer}, 7:536--541, 1930.

\bibitem{coifman1995translation}
Ronald~R Coifman and David~L Donoho.
\newblock {\em Translation-invariant de-noising}.
\newblock Springer New York, 1995.

\bibitem{dabov2009bm3d}
Kostadin Dabov, Alessandro Foi, Vladimir Katkovnik, Karen Egiazarian, et~al.
\newblock {BM3D} image denoising with shape-adaptive principal component
  analysis.
\newblock In {\em SPARS'09-Signal Processing with Adaptive Sparse Structured
  Representations}, 2009.

\bibitem{elad2006image}
Michael Elad and Michal Aharon.
\newblock Image denoising via sparse and redundant representations over learned
  dictionaries.
\newblock {\em Image Processing, IEEE Transactions on}, 15(12):3736--3745,
  2006.

\bibitem{flannery1992numerical}
Brian~P Flannery, William~H Press, Saul~A Teukolsky, and William Vetterling.
\newblock Numerical recipes in {C}.
\newblock {\em Press Syndicate of the University of Cambridge, New York}, 1992.

\bibitem{foi2007pointwise}
Alessandro Foi, Vladimir Katkovnik, and Karen Egiazarian.
\newblock Pointwise shape-adaptive dct for high-quality denoising and
  deblocking of grayscale and color images.
\newblock {\em Image Processing, IEEE Transactions on}, 16(5):1395--1411, 2007.

\bibitem{golub2012matrix}
Gene~H Golub and Charles~F Van~Loan.
\newblock {\em Matrix computations}, volume~3.
\newblock JHU Press, 2012.

\bibitem{higham1997stable}
Nicholas~J Higham.
\newblock Stable iterations for the matrix square root.
\newblock {\em Numerical Algorithms}, 15(2):227--242, 1997.

\bibitem{hyvarinen2009natural}
Aapo Hyv{\"a}rinen, Jarmo Hurri, and Patrik~O Hoyer.
\newblock {\em Natural Image Statistics: A Probabilistic Approach to Early
  Computational Vision.}, volume~39.
\newblock Springer, 2009.

\bibitem{meyer2014perturbation}
Fran{\c{c}}ois~G Meyer and Xilin Shen.
\newblock Perturbation of the eigenvectors of the graph laplacian: Application
  to image denoising.
\newblock {\em Applied and Computational Harmonic Analysis}, 36(2):326--334,
  2014.

\bibitem{motta2011idude}
Giovanni Motta, Erik Ordentlich, Ignacio Ramirez, Gadiel Seroussi, and
  Marcelo~J Weinberger.
\newblock The idude framework for grayscale image denoising.
\newblock {\em Image Processing, IEEE Transactions on}, 20(1):1--21, 2011.

\bibitem{perona1990scale}
Pietro Perona and Jitendra Malik.
\newblock Scale-space and edge detection using anisotropic diffusion.
\newblock {\em Pattern Analysis and Machine Intelligence, IEEE Transactions
  on}, 12(7):629--639, 1990.

\bibitem{rudin1992nonlinear}
Leonid~I Rudin, Stanley Osher, and Emad Fatemi.
\newblock Nonlinear total variation based noise removal algorithms.
\newblock {\em Physica D: Nonlinear Phenomena}, 60(1):259--268, 1992.

\bibitem{singer2009diffusion}
Amit Singer, Yoel Shkolnisky, and Boaz Nadler.
\newblock Diffusion interpretation of nonlocal neighborhood filters for signal
  denoising.
\newblock {\em SIAM Journal on Imaging Sciences}, 2(1):118--139, 2009.

\bibitem{szeliski2008comparative}
Richard Szeliski, Ramin Zabih, Daniel Scharstein, Olga Veksler, Vladimir
  Kolmogorov, Aseem Agarwala, Marshall Tappen, and Carsten Rother.
\newblock A comparative study of energy minimization methods for markov random
  fields with smoothness-based priors.
\newblock {\em Pattern Analysis and Machine Intelligence, IEEE Transactions
  on}, 30(6):1068--1080, 2008.

\bibitem{tal1989polynomial}
Hillel Tal-Ezer.
\newblock Polynomial approximation of functions of matrices and applications.
\newblock {\em Journal of Scientific Computing}, 4(1):25--60, 1989.

\bibitem{trefethen2000spectral}
Lloyd~N Trefethen.
\newblock {\em Spectral methods in MATLAB}, volume~10.
\newblock SIAM, 2000.

\bibitem{trefethen2008gauss}
Lloyd~N Trefethen.
\newblock Is {Gauss} quadrature better than {Clenshaw-Curtis}?
\newblock {\em SIAM review}, 50(1):67--87, 2008.

\bibitem{zontak2011internal}
Maria Zontak and Michal Irani.
\newblock Internal statistics of a single natural image.
\newblock In {\em Computer Vision and Pattern Recognition (CVPR), 2011 IEEE
  Conference on}, pages 977--984. IEEE, 2011.

\end{thebibliography}


\begin{appendices}
\label{appendix}

\section{Proof of Lemma \ref{lemma:NLM_matrix_properties}} \label{sec:apn_proof_of_lemma}
Lemma \ref{lemma:NLM_matrix_properties} summarizes known properties of NLM operators. We provide its proof for the self-containedness of the paper.
\begin{proof}
The NLM operator of Definition \ref{def:nlmeans_op} is in general not symmetric, but it is conjugated via $D^{\frac{1}{2}}$ to the symmetric matrix $S = D^{-\frac{1}{2}}WD^{-\frac{1}{2}}$, namely $A=D^{-1/2} S D^{1/2}$. Therefore, the NLM operator is diagonalizable and has the same eigenvalues as $S$. On other hand, $K_h$ is the Gaussian kernel function, which implies that $W$ is a symmetric positive definite matrix. Since $S$ is obtained from $W$ by multiplication by $D^{-1/2}$ on both sides, by Sylvester's law of inertia, the eigenvalues of $S$ are positive as well. Thus, since $S$ and $A$ are conjugated, all eigenvalues of $A$ are also positive.

Since $A$ is element-wise positive, it follows from the Perron-Frobenius theorem that
\[ \min_{1 \le i \le n} \sum_{j=1}^n A_{ij} \le \abs{\max_{1 \le i \le n}(\lambda_i)} \le \max_{1 \le i \le n} \sum_{j=1}^n A_{ij} . \]
Moreover, since for any $1 \le i \le n$, $\sum_j {A_{ij}}=1$, we get that $\abs{\max_{1 \le i \le n}(\lambda_i)} =1$. Thus, $\lambda_{1}=1$ and from the positivity of the eigenvalues we conclude that $\lambda_i \in(0,1]$, $1 \le i \le n$.
\end{proof}

\section{Algorithms}  \label{sec:apx_algorithms}

This appendix contains the pseudocode for the algorithms described in the paper. In these algorithms we denote by $ \ChebCoef(f,N)$ the set of the first $N+1$ Chebyshev coefficients calculated via \eqref{eqn:chebyshev_series_scalars} and \eqref{eqn:Cheby_Coef}.

\begin{algorithm}
\begin{algorithmic}[1]
\Require Noisy image $Y$, patch size $p$, kernel width $h$, matrix approximation rank $k$.
\Ensure Denoised image $\hat X$.
\Procedure{NLM-Eig}{$Y, p, h, k$}
\State $A \gets \NLM_{p, h}(Y)$ \Comment{Create a NLM operator for the image.}
\State $B \gets \EIG(A, k)$ \Comment{Make a low-rank approximation of the NLM operator}.
\State $y \gets \COL(Y)$ \Comment{Convert the image $Y$ to a vector.}
\State $\hat x \gets By$ \Comment{Compute the output image using the low-rank approximation of $A$.}
\State \textbf{return} $\IMAGE(\hat x)$ \Comment{Reshape $\hat{x}$ as an image.}
\EndProcedure
\end{algorithmic}
\caption{NLM-Eig denoising scheme. Here, in the notation of Lemma \ref{lemma:NLM_matrix_properties}, $\EIG(A, k) = Q^{-1} \Lambda_k Q$,  where $\Lambda_k = \diag\left(\lambda_1,\ldots,\lambda_k,0,\ldots,0  \right)$.}
\label{alg:nlm_eig}
\end{algorithm}

\begin{algorithm}
\begin{algorithmic}[1]
\Require Matrix $A$, vector $y$, vector of coefficients $c$ of length $N$.
\Ensure The vector $\sum_{k=1}^{N}c_{k}T_{k}\left(A\right)y$.
\Procedure{ClenshawMatVec}{A, y, c}
\State $N\gets \abs{c}$
\State $T\gets 2 \cdot A - I_{n}$
\State $d\gets 0$
\State $dd\gets 0$
\For{$i\gets N \; downto \; 2$}
\State $temp\gets d $
\State $d\gets 2 \cdot T \cdot d - dd + c_i \cdot y$
\State $dd\gets temp$
\EndFor
\State \textbf{return} $T \cdot d - dd + 0.5 \cdot c_1\cdot y$
\EndProcedure
\end{algorithmic}
\caption{Evaluate the product of a truncated matrix Chebyshev expansion \eqref{eqn:cheb_mat_approx} by a vector.}
\label{alg:clenshaw}
\end{algorithm}

\begin{algorithm}
\begin{algorithmic}[1]
\Require Noisy image $Y$, patch size $p$, kernel width $h$, cutoff and order parameters $\omega$ and $d$ of $f^{sb}_{\omega,d}$, number of Chebychev coefficients $N$.
\Ensure Denoised image $\hat x$.
\Procedure{NLM-SB}{$Y, p, h, \omega, d,N$}
\State $\{ \alpha_j \} \gets \ChebCoef(f^{sb}_{\omega, d}, N)$ \Comment{Compute the Chebychev coefficients for $f^{sb}_{\omega, d}$.}
\State $A \gets \NLM_{p, h}(Y)$ \Comment{Create a NLM operator for the image.}
\State $x \gets \COL(Y)$ \Comment{Convert the image $Y$ to a vector.}
\State $\hat x \gets$ ClenshawMatVec($A, \{ \alpha_j \}, x$) \Comment{Algorithm~\ref{alg:clenshaw}}

\State \textbf{return} $\IMAGE(\hat x)$ \Comment{Reshape $\hat{x}$ as an image.}
\EndProcedure
\end{algorithmic}
\caption{NLM-SB denoising scheme.}
\label{alg:nlm_sbw}
\end{algorithm}

\begin{algorithm}
\begin{algorithmic}[1]
\Require Noisy image $Y$, patch size $p$, kernel widths $h_1$ and $h_2$, mixing weight $\gamma \in [0,1]$, cutoff and order parameters $\omega_1, \omega_2$ and $d_1, d_2$ of $f^{sb}_{\omega,d}$, number of Chebychev coefficients $N$.
\Ensure Denoised image $\hat x^{(3)}$.
\Procedure{NLM-SB2}{$Y, p, \gamma, h_1, h_2, \omega_1, \omega_2, d_1, d_2$}
\State $\hat x^{(1)} \gets \NLMSB(Y, p, h_{1}, \omega_{1}, d_{1} ,N$) \Comment{Algorithm~\ref{alg:nlm_sbw}.}
\State $\hat x^{(2)}\gets (1-\gamma)\hat x^{(1)} + \gamma x$ \Comment{Mix the esimated image with the original one.}
\State $\hat x^{(3)} \gets \NLMSB(\IMAGE(x^{(2)}), p, h_{2}, \omega_{2}, d_{2} ,N$) \Comment{Denoise again.}

\State \textbf{return} $\hat x^{(3)}$
\EndProcedure
\end{algorithmic}
\caption{NLM-SB2 two-stage denoising scheme.}
\label{alg:sbw_twostage}
\end{algorithm}

\end{appendices}

\end{document}